\pdfoutput=1

\documentclass{article}

\usepackage[final,nonatbib]{neurips_2021}

\usepackage[utf8]{inputenc} 
\usepackage[T1]{fontenc}    
\usepackage{hyperref}       
\usepackage{url}            
\usepackage{booktabs}       
\usepackage{amsfonts}       
\usepackage{amsmath}
\usepackage{amsthm}
\usepackage{amssymb}
\usepackage{mathtools}
\usepackage{bbm}
\usepackage{nicefrac}       
\usepackage{microtype}      
\usepackage{xcolor}         

\usepackage{enumitem}

\def\eg{\emph{e.g.}}
\def\ie{\emph{i.e.}}
\def\Hcal{{\mathcal H}}

\def\Ncal{{\mathcal N}}

\def\kmone{{k\text{--}1}}

\def\dmone{{d\text{--}1}}

\def\R{{\mathbb R}}

\def\Sbb{{\mathbb S}}

\def\Nbar{{\overline N}}
\def\Ybar{{\overline Y}}
\def\Vbar{{\overline V}}
\def\Kbar{{K_G}}
\def\Sym{{\mathcal S}}

\newcommand{\EE}[2]{\mathbb{E}_{#1}\left[#2\right] }

\DeclarePairedDelimiter\ceil{\lceil}{\rceil}
\DeclarePairedDelimiter\floor{\lfloor}{\rfloor}
\DeclareMathOperator{\Tr}{Tr}

\DeclareMathOperator{\E}{\mathbb{E}}
\DeclareMathOperator{\1}{\mathbbm{1}}

\newtheorem{theorem}{Theorem}
\newtheorem{proposition}[theorem]{Proposition}
\newtheorem{lemma}[theorem]{Lemma}
\newtheorem{corollary}[theorem]{Corollary}
\newtheorem{example}[theorem]{Example}

\title{On the Sample Complexity of Learning \\ under Invariance and Geometric Stability}

\author{%
  Alberto Bietti \\
  NYU\thanks{Center for Data Science, New York University.} \\
  \texttt{alberto.bietti@nyu.edu} \\
  \And Luca Venturi \\
  NYU\thanks{Courant Institute for Mathematical Sciences, New York University.} \\
  \texttt{lv800@nyu.edu} \\
  \And Joan Bruna \\
  NYU\thanks{Center for Data Science and Courant Institute for Mathematical Sciences, New York University.} \\
  \texttt{bruna@cims.nyu.edu} \\
}

\begin{document}

\maketitle

\begin{abstract}

Many supervised learning problems involve high-dimensional data such as images, text, or graphs. In order to make efficient use of data, it is often useful to leverage certain geometric priors in the problem at hand, such as invariance to translations, permutation subgroups, or stability to small deformations. 
We study the sample complexity of learning problems where the target function presents such invariance and stability properties, by considering spherical harmonic decompositions of such functions on the sphere. We provide non-parametric rates of convergence for kernel methods, and show improvements in sample complexity by a factor equal to the size of the group when using an invariant kernel over the group, compared to the corresponding non-invariant kernel. These improvements are valid when the sample size is large enough, with an asymptotic behavior that depends on spectral properties of the group. Finally, these gains are extended beyond invariance groups to also cover geometric stability to small deformations, modeled here as subsets (not necessarily subgroups) of permutations. 

\end{abstract}

\section{Introduction}
\label{sec:introduction}

Learning from high-dimensional data is known to be statistically intractable without strong assumptions on the problem.
A canonical example is learning Lipschitz functions, which generally requires a number of samples exponential in the dimension due to the curse of dimensionality (\eg,~\cite{wainwright2019high}).
Many high-dimensional machine learning problems involve highly structured data such as images, text, or graphs, and may exhibit invariance to certain transformations of the input data, such as permutations, translations or rotations, and near invariance to small deformations. More precisely, if $\mathcal{X}$ is 
the high-dimensional data domain, and $G$ is a set of transformations $\sigma: \mathcal{X} \to \mathcal{X}$, the learning task can be alleviated if one knows in advance that the target function $f$ varies smoothly under transformations in $G$: $|f( \sigma \cdot x) - f(x)|$ is uniformly small over $x \in \mathcal{X}$ for~$\sigma \in G$.

To further motivate this property, it is useful to view the data domain $\mathcal{X}$ as a space of signals $\mathcal{X} = L^2(\Omega;\mathbb{R})$ defined over a geometric domain $\Omega$, such as a 2d grid. The set of transformations $G$ can then be articulated in terms of $\Omega$ rather than $\mathcal{X}$, a much simpler geometric object, and then \emph{lifted} into $\mathcal{X}$ by composition: if $\sigma: \Omega \to \Omega$, and $x\in \mathcal{X}$ then $(\sigma \cdot x)(u):= x( \sigma^{-1}(u) )$ for every $u \in \Omega$. The smoothness to transformations can thus be interpreted as a form of  \emph{geometric stability}.

In this paper, we quantify the sample complexity gains brought by geometric stability. Concretely, we consider target functions~$f$ defined on the sphere~$\mathcal{X}=\Sbb^{d-1}$ in~$d$ dimensions with finite~$L^2(\Sbb^{d-1})$ norm. 
In this case, we view the geometric domain as the discrete 1d grid $\Omega=[1,\ldots, d]$, 
and consider geometric transformations $G$ as subsets of the symmetric group of permutations of $d$ elements. Given a set  $G$ (not necessarily a group), we consider the \emph{smoothing} operator given by~$S_G f(x) = \frac{1}{|G|} \sum_{\sigma \in G} f(\sigma \cdot x)$
for~$f \in L^2(\Sbb^{d-1})$,
and assume that our target function $f$ is geometrically stable, in the sense that $f=S_G g$ for some $g \in L^2(\Sbb^{d-1})$. In words, the smoothing operator $S_G$ replaces the prediction $f(x)$ by the average over transformations of~$x$.
In particular, functions which are invariant under the action of $\sigma \in G$, namely
\begin{equation}
\label{eq:group_invariance}
f(\sigma \cdot x) = f(x), \quad \sigma \in G, x \in \Sbb^{d-1},
\end{equation}
are also stable, with $f = S_G f$.

 Building on the recent work \cite{mei2021learning},  
we proceed by studying harmonic decompositions of such functions using spherical harmonics~\cite{costas2014spherical}, which generalize Fourier series on the circle to higher dimensions.
This allows us to obtain rates of approximation for invariant and geometrically stable functions with varying levels of smoothness, and to study the generalization properties of invariant kernel methods using kernels defined on the sphere. Specifically, our main contributions are:
\begin{itemize}[leftmargin=0.4cm,topsep=0pt]
    \item By comparing spectral properties of usual kernels on the sphere with invariant ones, we find that the latter provide improvements in sample complexity by a factor of the order of the size of the group when the sample size is large enough (Section \ref{sec:stat}).
    \item We study how this improvement factor varies with sample size, in terms of the structure of the group and on spectral properties of the permutation matrices it contains (Section \ref{sec:approx}). 
    \item We extend the invariance analysis to geometrically stable functions, establishing similar gains in sample complexity that depend on the size of the transformation subset (Section \ref{sec:stability}). 
\end{itemize}

Our proofs rely on comparing the dimension of invariant and non-invariant spherical harmonics at a given degree, 
and showing that their ratio decays to the inverse group size as the degree tends to infinity.
In contrast to~\cite{mei2021learning}, we consider the dimension to be fixed and study non-parametric rates of convergence for potentially non-smooth target functions and general groups of permutations, while they consider a different regime in high dimension and focus on invariance to translation groups.

\paragraph{Related work.}
Invariance and deformation stability have been analysed using convolutional neural network-type architectures such as the scattering transform \cite{mallat2012group,bruna2013invariant}, or convolutional kernels \cite{bietti2019group, mairal2014convolutional}. While these works characterise the stability in terms of the dyadic structure of convolutional filters (such as wavelets), they do not cover a statistical analysis of sample complexity. Similarly, models of compositional functions such as those in \cite{cohen2016inductive,mhaskar2016deep, poggio2017and} study the benefit of hierarchical representations with local connectivity for approximation, while~\cite{li2020convolutional,malach2020computational} study benefits of local connectivity with optimization-based algorithms; yet these works do not consider invariance or stability. 
\cite{mei2021learning} studies similar benefits of invariance but in a different, high-dimensional, regime where only polynomials can be learned, focusing on translation groups, while we consider arbitrary groups or sets of permutations in fixed dimension. \cite{elesedy2021provably} also sudies benefits of group invariance, but focuses on linear models, and only considers interpolating estimators. \cite{sokolic2017generalization} study general generalization bounds of invariant classifiers that scale exponentially with the dimension, which would be pessimistic under our assumptions. \cite{ciliberto2019localized} study benefits of equivariant kernels in structured prediction problems. \cite{du2018many} studies generalisation advantages of CNNs over fully-connected models, while our focus is on non-parametrics.

\section{Preliminaries}
\label{sec:background}

In this section, we describe our setup and provide some background on harmonic decompositions on the sphere, and how these are affected by invariance.

\paragraph{Statistical learning setup.}
We consider a supervised learning problem where the data distribution $\rho$ on input-label pairs~$(x, y)$ is such that~$x \in \Sbb^{d-1}$ and~$\E[y|x] = f^*(x)$ for some target function~$f^*$in~$L^2(\Sbb^{d-1})$. For simplicity, we will assume that~$x$ is uniformly distributed on the sphere, and denote the uniform measure on~$\Sbb^{d-1}$ by~$d\tau$.
We consider a regression setup with $L^2$ risk given by
\begin{equation*}
R(f) = \EE{(x,y) \sim \rho}{(f(x) - y)^2}.
\end{equation*}
For a given estimator~$\hat f_n$ based on~$n$ samples from~$\rho$, the goal is then to obtain generalization bounds as a function of~$n$ on the excess risk
\begin{equation}
\E[R(\hat f_n)] - R(f^*) = \E[\|\hat f_n - f^*\|^2_{L^2(d \tau)}],
\end{equation}
where the expectation is over the~$n$ samples.
Such bounds are well-studied for various classes of target functions~$f^*$ such as smoothness classes, and estimators such as kernel ridge regression.
These are typically studied through harmonic decompositions of~$f^*$ and of a kernel function in appropriate~$L^2$ bases, which then relate function regularity and decays of Fourier coefficients.

\paragraph{Harmonic analysis on the sphere.}
When considering functions in~$L^2(d \tau)$, an appropriate choice of orthonormal basis is that of spherical harmonic polynomials~\cite{atkinson2012spherical,costas2014spherical}.
More precisely, denote~$\{Y_{k,j}\}_{j=1}^{N(d,k)}$ denote an orthonormal basis of the space~$V_{d,k}$ of spherical harmonics of degree~$k$, \ie, homogeneous harmonic polynomials of degree~$k$, where~$N(d,k) = \frac{2k + d - 2}{k} {k + d - 3 \choose d - 2}$.
Then, the collection~$\{Y_{k,j} : k \geq 0, j = 1, \ldots, N(d,k)\}$ forms an orthonormal basis of~$L^2(d \tau)$, so that any function~$f \in L^2(d \tau)$ may be written
\begin{equation}
f(x) = \sum_{k \geq 0} \sum_{j=1}^{N(d,k)} a_{k,j} Y_{k,j}(x),
\end{equation}
with~$\sum_k \sum_{j = 1}^{N(d,k)} a_{k,j}^2 < \infty$.
Similarly, any dot-product kernel $K(x,x') = \kappa(\langle x, x' \rangle)$ on the sphere may be written
\begin{equation}
\label{eq:kappa_decomp}
\kappa(\langle x, x' \rangle) = \sum_{k \geq 0} \mu_k \sum_{j=1}^{N(d,k)} Y_{k,j}(x) Y_{k,j}(x'),
\end{equation}
where~$\mu_k$ is given by~$\mu_k = \frac{\omega_{d-2}}{\omega_{d-1}} \int_{-1}^1 \kappa(t) P_{d,k}(t) (1 - t^2)^{(d-3)/2} dt$.
Here,~$\omega_{p-1}$ is the surface measure of the sphere in~$p$ dimensions, and $P_{d,k}$ are Legendre or Gegenbauer polynomials of degree~$k$ in~$d$ dimensions (normalized with~$P_{d,k}(1) = 1$),
which form an orthogonal basis of~$L^2([-1, 1], d w)$, with~$d w(t) = (1 - t^2)^{(d-3)/2} dt$.
When the kernel~$K$ is positive definite and is used in the context of kernel ridge regression with data uniformly distributed on the sphere, then the~$\mu_k$ also correspond to the eigenvalues of the covariance operator.
These eigenvalues and their decay then control the statistical properties of the kernel ridge regression estimator~\cite{caponnetto2007optimal}.

\paragraph{Spherical harmonics and group-invariant functions.}
In order to describe harmonic decompositions of functions satisfying the group invariance property~\eqref{eq:group_invariance} for a discrete group~$G$, we follow~\cite{mei2021learning} and define the symmetrization operator
\begin{equation}
S_G f(x) = \frac{1}{|G|} \sum_{\sigma \in G} f(\sigma \cdot x).
\end{equation}
This operator acts as a projection from~$L^2(d \tau)$ to a subset thereof which contains invariant functions.
It can be shown that the spaces~$V_{d,k}$ of spherical harmonics of degree~$k$ are stable by~$S_G$~\cite{mei2021learning}, and we may then define an orthonormal basis of~$\Vbar_{d,k} := S_G V_{d,k}$ consisting of \emph{invariant} spherical harmonics~$\{\Ybar_{k,j}\}_{j=1}^{\Nbar(d,k)}$.
We then have the following lemma.

\begin{lemma}[Representation of projection~\cite{mei2021learning}]
\label{lemma:projection}
For any~$k\geq 0$, we have
\begin{equation}
\label{eq:gammadk}
\gamma_d(k) := \frac{\Nbar(d,k)}{N(d,k)} = \frac{1}{|G|} \sum_{\sigma \in G} \EE{x}{P_{d,k}(\langle \sigma \cdot x, x \rangle)}.
\end{equation}
\end{lemma}

The quantity~$\gamma_d(k)$ will play an important role in determining the gains in sample complexity brought by invariance. We will show in Section~\ref{sec:approx} that~$\gamma_d(k)$ converges to~$1/|G|$ for large~$k$, with an asymptotic behavior that is governed by spectral properties of the elements of the group.

\section{Sample Complexity of Invariant Kernels}
\label{sec:stat}

We begin our study by focusing on the invariant case. 
In this section, we study the sample complexity of learning invariant functions, by considering kernel ridge regression estimators and providing non-parametric rates of convergence that illustrate the gains achievable with invariant kernels compared to non-invariant ones.

\paragraph{Kernel ridge regression (KRR) and invariant kernels.}
For a positive definite kernel~$K$ with RKHS~$\Hcal_K$, we consider the KRR estimator~$\hat f_\lambda$ given by
\begin{equation}
\hat f_\lambda := \arg\min_{f \in \Hcal_K} \frac{1}{n}  \sum_{i=1}^n (f(x_i) - y_i)^2 + \lambda \|f\|_{\Hcal_K}^2.
\end{equation}
We consider the following kernels, which we assume positive definite, given for~$x, x' \in \Sbb^{d-1}$ by
\begin{equation}
\label{eq:kernels}
K(x, x') = \kappa(\langle x, x' \rangle), \qquad \Kbar(x, x') = \frac{1}{|G|} \sum_{\sigma \in G} \kappa(\langle \sigma \cdot x, x' \rangle),
\end{equation}
with~$\kappa(u) \leq 1$. A common example for~$\kappa$ is the arc-cosine kernel~\cite{cho2009kernel}, which arises from infinite-width shallow neural networks with ReLU activations.
The following integral operator defined on~$L^2(d \tau)$ and its eigen decomposition play an important role for the statistical and approximation properties of kernel methods:
\begin{equation}
\label{eq:integral_op}
T_K f(x) = \int K(x, x') f(x') d \tau(x').
\end{equation}
We now show that its spectral properties are closely related for~$K$ and~$\Kbar$.
\begin{lemma}[Spectral properties of~$K$ and~$\Kbar$.]
\label{lemma:eigenvalues}
There exists a basis of spherical harmonics in which the operators~$T_K$ and~$T_\Kbar$ are jointly diagonalized.
They admit the same eigenvalues~$\mu_k$ as in~\eqref{eq:kappa_decomp}, with multiplicity~$N(d,k)$ for~$T_K$ and~$\Nbar(d,k)$ for~$T_\Kbar$.
\end{lemma}
The decay of the eigenvalues~$\mu_k$ controls the smoothness of functions in the RKHS, for instance when~$\mu_k$ decays polynomially, $T_K$ behaves similarly to powers of the Laplacian on the sphere, leading to functional spaces similar to Sobolev spaces.
For the example of the arc-cosine kernel, $\mu_k$ decays as~$k^{-d-2}$. leading to an RKHS containing functions with~$d/2+1$ bounded derivatives~\cite{bach2017breaking}.

\paragraph{Approximation error.}
The approximation error of kernel methods is often controlled by the following quantity (\eg,~\cite{bach2021ltfp,cucker2002mathematical}):
\begin{equation}
\label{eq:approx_error}
A_{\Hcal}(\lambda, f^*) = \inf_{f\in \Hcal} \|f - f^*\|^2_{L^2(d \tau)} + \lambda \|f\|^2_\Hcal,
\end{equation}
where~$f^*$ is a target function in~$L^2(d \tau)$, and~$\Hcal$ is a given RKHS.
In particular, if~$f^*$ is smooth enough so that~$f^* \in \Hcal$, then we have~$A_\Hcal(\lambda,f^*) \leq \lambda \|f^*\|^2_\Hcal$,
while if~$f^* \notin \Hcal$, \eg, if~$f^*$ is only Lipschitz, then~$A_\Hcal(\lambda,f^*)$ typically grows much faster with~$\lambda$.
We now show a useful result for invariant targets, showing that in this case the approximation error is the same for the kernels~$K$ and~$\Kbar$.

\begin{lemma}[Approximation error for invariant functions.]
\label{lemma:approx_error}
If~$f^*$ is invariant to the group~$G$, so that~$f^* = S_G f^*$, then we have
\begin{equation}
A_{\Hcal_K}(\lambda, f^*) = A_{\Hcal_{\Kbar}}(\lambda,f^*).
\end{equation}
\end{lemma}

\paragraph{Degrees of freedom.}
The above result suggests that any gains of using~$\Kbar$ instead of~$K$ for learning invariant functions should come from estimation rather than approximation error.
The estimation error of ridge rigression estimators is typically controlled with the following quantity, often called \emph{degrees of freedom} or \emph{effective dimension} (\eg,~\cite{bach2021ltfp,hastie2009elements}):
\begin{equation}
\Ncal_{K}(\lambda) = \Tr(\Sigma_K(\Sigma_K + \lambda I)^{-1}) = \sum_{m \geq 0} \frac{\lambda_m}{\lambda_m + \lambda},
\end{equation}
where~$\Sigma_K = \EE{x}{K(x,\cdot) \otimes_{\Hcal_K} K(x,\cdot)}$ is the covariance operator and~$(\lambda_m)_{m \geq 0}$ its eigenvalues, taking multiplicity into account, which are the same as those of~$T_K$ when data is distributed according to~$d \tau$~\cite{caponnetto2007optimal}.
We then obtain the following simple result relating~$\Ncal_{\Kbar}$ to~$\Ncal_K$.

\begin{lemma}[Degrees of freedom for~$K$ and~$\Kbar$.]
\label{lemma:dof}
For any~$\ell \geq 0$, we have
\begin{equation*}
\Ncal_{\Kbar}(\lambda) \leq D(\ell) + \nu_d(\ell) \Ncal_K(\lambda),
\end{equation*}
where~$D(\ell) := \sum_{k < \ell} \Nbar(d,k)$ and~$\nu_d(\ell) := \sup_{k \geq \ell} \gamma_d(k)$, with~$\gamma_d$ given in~\eqref{eq:gammadk}.
\end{lemma}
This suggests that for a fixed~$\ell$, the effective dimension of~$\Kbar$ is controlled by a factor~$\nu_d(\ell)$ times that of~$K$, up to a finite fixed dimension~$D(\ell)$.
For difficult non-parametric problems which require small~$\lambda$ at large sample sizes, the second term will tend to dominate, so that having a small~$\nu_d(\ell)$ may help reduce sample complexity compared to using the vanilla kernel~$K$, an observation which we make rigorous below.

\paragraph{Generalization bound for KRR.}
Armed with the above lemmas on approximation error and degrees of freedom, we now study generalization of KRR under the following assumptions:
\begin{itemize}[leftmargin=1cm,topsep=0pt]
	\item[(A1)] \emph{capacity condition}: $\Ncal_K(\lambda) \leq C_K \lambda^{-1/\alpha}$ with~$\alpha > 1$.
	\item[(A2)] \emph{source condition}: there exists~$r > \frac{\alpha - 1}{2 \alpha}$ and~$g \in L^2(d \tau)$ with~$\|g\|_{L^2(d \tau)} \leq C_{f^*}$ such that~$f^* = T_K^r g$.
	\item[(A3)] \emph{invariance}: $f^*$ is~$G$-invariant.
	\item[(A4)] \emph{problem noise}: $\rho$ is such that $\E_\rho[(y - f^*(x))^2|x] \leq \sigma_\rho^2$.
\end{itemize}
The first, second, and fourth conditions are commonly used in the kernel methods literature~\cite{caponnetto2007optimal}.
(A1) characterizes the ``size'' of the RKHS, and is satisfied when the eigenvalues~$\lambda_m$ of~$T_K$ decay as~$k^{-\alpha}$.
On the sphere, $\alpha = \frac{2s}{d-1}$ corresponds to having~$s$ bounded derivatives, \eg, we have~$s = d/2 + 1$ for the arc-cosine kernel.
The parameter~$r$ in~(A2) defines the regularity of~$f^*$ relative to that of the kernel: $r = 1/2$ corresponds to~$f^* \in \Hcal_K$, while larger (resp.~smaller)~$r$ implies~$f^*$ is more (resp.~less) smooth.
The condition on~$r$ is needed for our specific bound, which is based on~\cite[Proposition 7.2]{bach2021ltfp},
but may be bypassed using different algorithms or analyses~\cite{fischer2020sobolev,pillaud2018statistical}.
We now present our bound on the excess risk.
\begin{theorem}[Generalization of invariant kernel.]
\label{thm:generalization}
Assume (A1-4). Let $\nu_d(\ell)$ be as in Lemma~\ref{lemma:dof}, or an upper bound thereof, and assume~$\nu_0 := \inf_{\ell \geq 0} \nu_d(\ell) > 0$. \\
Let~$n \geq \max \left\{ \|f^*\|_\infty^2 / \sigma_\rho^2, \left( C_1/\nu_0 \right)^{\frac{\alpha}{2 \alpha r + 1 - \alpha}} \right\}$,
and define
\begin{equation}
\label{eq:ln_equation}
    \ell_n := \sup \{\ell : D(\ell) \leq C_2 \nu_d(\ell)^{\frac{2 \alpha r}{2 \alpha r + 1}} n^{\frac{1}{2 \alpha r + 1}}\}.
\end{equation}
We then have,  for~$\lambda = C_3 (\nu_d(\ell_n) / n)^{\alpha/(2 \alpha r + 1)}$,
\begin{equation}
\label{eq:invariant_bound}
\E[R(\hat f_\lambda) - R(f^*)] \leq C_4 \left(\frac{ \nu_d(\ell_n)}{n}\right)^{\frac{2 \alpha r}{2 \alpha r + 1}}.
\end{equation}
In the same setting, KRR with kernel~$K$ and~$\lambda = C_3 n^{\frac{-\alpha}{2 \alpha r + 1}}$ achieves~$\E[R(\hat f_\lambda) - R(f^*)] \leq C_4 n^{\frac{-2 \alpha r}{2 \alpha r + 1}}$,
where~$C_3, C_4$ are the same constants as for the invariant kernel.
Here, the constants~$C_{1:4}$ only depend on the parameters of assumptions (A1-4).
\end{theorem}

The theorem shows that the generalization error for the invariant kernel behaves as if it effectively had access to $n / \nu_d(\ell_n)$ samples, so that~$\nu_d(\ell_n)$ plays the role of an \emph{effective} inverse sample complexity gain at sample size~$n$.
Note that~$\nu_d(\ell_n) \leq 1$ and~$\nu_d$ is decreasing, so that we always have some improvement in sample complexity, and this gets better when~$\ell_n$ is large.
In particular, we show in Section~\ref{sec:approx} that~$\gamma_d(k)$, and hence~$\nu_d(\ell)$ converge to~$1/|G|$, so that asymptotically the gain in sample complexity can be as large as the size of the group, which in some cases may grow \emph{exponentially} in $d$. 

\paragraph{Asymptotic estimates of the effective gain~$\nu_d(\ell_n)$.}
We now study the asymptotic behavior of the effective gain factor~$\nu_d(\ell_n)$ when~$n \to \infty$, by considering a case where an asymptotic equivalent of~$\nu_d(\ell)$ in~$\ell$ is known:
\begin{equation*}
    \nu_d(\ell) \approx \nu_0 + c \ell^{-\beta}.
\end{equation*}
In Section~\ref{sec:approx}, we obtain such asymptotics with~$\nu_0 = 1/|G|$, and a rate~$\beta$ that depends on spectral properties of the elements of~$G$, as well as upper bounds with possibly faster rates~$\beta$ at the cost of larger~$\nu_0$.
In Appendix~\ref{sub:appx_ell_n}, we show that we may leverage this to obtain the following asymptotic estimate of the effective gain~$\nu_d(\ell_n)$:
\begin{equation}
\label{eq:nud_elln}
    \nu_d(\ell_n) \leq \nu_0 + C \min\left\{ (\nu_0^{2 \alpha r} n)^{\frac{-\beta}{(d-1)(2 \alpha r + 1)}}, n^{\frac{-\beta}{(d-1)(2 \alpha r + 1) + 2 \beta \alpha r}}\right\}.
\end{equation}
Notice that when~$\beta \ll d$, both exponents of~$n$ display a curse of dimensionality, but this curse goes away as~$\beta$ grows.
Note also that the first exponent yields a faster rate, but one that is only achieved for large~$n$ due to the factor~$\nu_0^{2 \alpha r}$, which may be small for large groups.

\paragraph{Curse of dimensionality and optimality.}
Note that the bound obtained in Theorem~\ref{thm:generalization} is still cursed for an invariant target~$f^*$, in the sense that the exponent in the rate is of order~$1/d$ when~$f^*$ is only assumed to be Lipschitz.
Indeed, a Lipschitz assumption on~$f^*$ corresponds to taking~$r$ and~$\alpha$ such that~$2 \alpha r \approx 2/(d-1)$, which makes the source condition (A2) similar to a bound on~$\|\nabla f^*\|_{L^2(d \tau)}$.
This then leads to a cursed rate $n^{-2/(2 + d-1)}$, raising the question of whether this can be improved.
We note that since $\gamma_d(k) = \Omega(1/|G|)$ (as we show in Section~\ref{sec:approx}), the asymptotic decays (as a function of~$k$) of the coefficients of~$f^*$ and of the eigenvalues of~$T_\Kbar$ are similar to those for the non-invariant case, which implies that these rates cannot be improved (see, \eg,~\cite{caponnetto2007optimal}).
In Appendix~\ref{sub:appx_optimality}, we show that our bounds with an improvement in sample complexity by a factor~$|G|$ are asymptotically minimax optimal,
so that this may be the best we can hope for under our assumptions.

\paragraph{Comparison to~\cite{mei2021learning}.}
The work~\cite{mei2021learning} also consider non-parametric learning of invariant functions with similar kernels.
They consider a high-dimensional regime where~$d \to \infty$ with sample sizes in polynomial scalings~$n \approx d^s$ for some~$s$.
They then show that if~$\gamma_d(k) = \Theta_d(d^{-\alpha})$ as~$d \to \infty$, for some~$\alpha > 0$ (which they call \emph{degeneracy}), then the invariant kernel can learn polynomials of degree~$\ell$ with~$n \approx d^{\ell - \alpha}$ while the non-invariant kernel needs~$n \approx d^\ell$ samples.
In some cases, such as the cyclic group, \cite{mei2021learning} show~$\alpha = 1$ and hence the gain of a factor~$d^\alpha = d$ is equal the size of the group, but in other cases~$d^\alpha$ may be smaller than the group size.
For groups of size exponential in~$d$, the analysis in~\cite{mei2021learning} may only achieve polynomial improvements by factors~$d^\alpha$, in contrast to our analysis, which considers the different regime of fixed~$d$ and~$n \to \infty$, and may lead to gains by exponential factors if~$|G|$ is large, at least asymptotically.

\section{Counting Invariant Polynomials}
\label{sec:approx}

In this section, we study the asymptotic behavior of~$\gamma_d(k)$, given in~\eqref{eq:gammadk}, when~$k \to \infty$ and the dimension~$d$ and the group~$G$ are fixed.
This quantity can be seen as capturing the fraction of orthogonal spherical harmonics of degree~$k$ that are invariant to~$G$, and helps us control the possible gains in sample complexity for learning invariant functions, as described in Section~\ref{sec:stat}.
Denoting~$\gamma_{d,\sigma}(k) := \EE{x}{P_{d,k}(\langle \sigma \cdot x, x \rangle)}$, we will show that~$\gamma_{d,\sigma}(k)$ vanishes for large~$k$ for any~$\sigma$ that is not the identity.
This implies that~$\gamma_d(k)$ converges to~$1/|G|$, since we trivially have~$\gamma_{d,Id}(k) = P_{d,k}(1) = 1$.
We further characterize the asymptotic behavior of~$\gamma_d(k)$ in terms of properties of the group elements. 
In the following we consider the case of $G$ being a subgroup of $S_d$, the groups of permutations on $d$ elements. 

\paragraph{Decay of~$\gamma_{d,\sigma}(k)$.}
Our main insight is to leverage the fact that when~$\sigma$ is not the identity, then the random variable~$z_\sigma = \langle \sigma \cdot x, x \rangle$ when~$x \sim \tau$ admits a density on $[-1, 1]$, which we denote~$q_\sigma$.
This by itself will prove sufficient to show that~$\gamma_{d,\sigma}(k)$ decays for large~$k$, thanks to the oscillatory behavior of~$P_{d,k}$.
We can then further characterize its asymptotic behavior by studying the singularities of~$q_\sigma$,
leveraging the seminal work of Saldanha and Tomei~\cite{saldanha1996accumulated}.
In particular, these depend on spectral properties of the matrix associated to $\sigma$.
We summarize this in the next proposition.

\begin{proposition}[Asymptotic behavior of~$\gamma_{d,\sigma}(k)$.]
\label{prop:decay}
Let
$A_\sigma$ 
be the matrix associated to $\sigma \neq \mathrm{Id}$, that is such that $\sigma\cdot x = A_\sigma x$.
Denote by~$\Lambda_\sigma$ the set of (complex) eigenvalues of~$A_\sigma$, and by~$m_\lambda$ the multiplicity of~$\lambda \in \Lambda_\sigma$.
When~$k \to \infty$, we have the asymptotic equivalent~$\gamma_{d,\sigma}(k) = \sum_{\lambda \in \Lambda_\sigma} \gamma_{d,\sigma,\lambda}(k)$, where
\begin{equation}
\gamma_{d,\sigma,\lambda}(k) \lesssim \begin{cases}
k^{-d + m_\lambda} + o(k^{-d + m_\lambda})~, &\text{ if }\lambda \in \{\pm 1\}~,\\
k^{-d + 
m_\lambda + 4
}
+ o(k^{-d +
m_\lambda +4
})~, &\text{ otherwise},
\end{cases}
\end{equation}
where~$\lesssim$ hides constants that may depend on~$d$, $\sigma$ and~$\lambda$. 
\end{proposition}

Every permutation $\sigma \in \Sym_d$ (where~$\Sym_d$ is the symmetric group of permutations) can be decomposed into cycles on disjoint orbits; the eigenvalues $\lambda$ (and their multiplicities $m_\lambda$) of a matrix $A_\sigma$ admit an interpretation based on such decomposition. Indeed, since $A_\sigma$ is unitary, its eigenvalues are of the form $\lambda = e^{2\pi i \theta}$, and one can verify that necessarily $\theta=\frac{p}{q} \in \mathbb{Q}$. Furthermore, assuming w.l.o.g.~that $q$ is prime, such eigenvalue appears whenever $\sigma$ contains a cycle of length a multiple of $q$. 
In particular, the multiplicity of the eigenvalue $1$, $m_1$, corresponds to the total number of cycles in such a decomposition, which we will denote by $c(\sigma)$. Then~$\gamma_{d,\sigma}(k)$ can be controlled as follows. 

\begin{corollary}[Decay of~$\gamma_{d,\sigma}(k)$]\label{corollary:gamma_d_sigma_decay}
Let $\sigma \ne \mathrm{Id}$, and let $c(\sigma)$ denote the number of cycles in $\sigma$. Then,
\begin{equation*}
\gamma_{d,\sigma}(k) \lesssim \begin{cases}
k^{-d + c(\sigma)} ~, &\text{ if } c(\sigma) > \frac{d+3}{2}~,\\
k^{-d/2 + 6}~, &\text{otherwise}~.
\end{cases}
\end{equation*}
\end{corollary}

\paragraph{Decay on specific subgroups.}
We may now use Corollary~\ref{corollary:gamma_d_sigma_decay} to study the asymptotic behaviour of $\gamma_d(k)$ as $k\to\infty$ for various choices of subgroups of $\Sym_d$, using the following result.

\begin{corollary}[Upper bounds with permutation statistics]\label{theo:number_of_cycles}
Let $G$ be a subset of $\Sym_d$ and define
$
\zeta(G, s) := \left|  \{ \sigma \in G \;:\; c(\sigma) > s \} \right|
$
for $s\in[d-1]$. Then, for any~$s$, we have 
\begin{equation}
\label{eq:gamma_upper_bound}
\gamma_d(k) \leq \frac{\zeta(G,s)}{|G|} + O\left( k^{-d + \max\{ s, \;d/2 + 6 \}} \right)~,
\end{equation}
with equality if $s$ is such that~$\zeta(G,s) = 1$.
\end{corollary}
Note that such an upper bound immediately yields a similar upper bound for~$\nu_d(\ell)$ as defined in Section~\ref{sec:stat}, which then controls the effective gain in sample complexity in Theorem~\ref{thm:generalization}.
Indeed, \eqref{eq:gamma_upper_bound} implies that there is a constant~$C$ such that for all~$k > 0$, $\gamma_d(k) \leq \zeta(G,s) / |G| + C k^{-d + \max\{s, d/2 + 6\}}$.
Since this upper bound decreases with~$k$, we obtain
\begin{equation*}
    \nu_d(\ell) \leq \frac{\zeta(G,s)}{ |G|} + C \ell^{-d + \max\{s, d/2 + 6\}}.
\end{equation*}
In the context of the generalization bound of Theorem~\ref{thm:generalization} and our heuristic derivation thereafter, the \emph{effective} gain in sample complexity is then governed an upper bound on~$\nu_d(\ell_n)$ as in~\eqref{eq:nud_elln},
with asymptotic gain~$\nu_0 = \frac{\zeta(G,s)}{|G|}$ and rate~$\beta = d - \max\{s, d/2 + 6\}$.

\begin{example}[Translations]
\label{ex:transl}
Let $G= C_d$ be the cyclic group on $d$ elements. Then it holds
$$
\gamma_d(k) = \frac{1}{d} + O\left( k^{-d/2 + 6} \right)~.
$$
This follows by noticing that every translation $\sigma$ (but the identity) satisfies $c(\sigma) \leq d/2$.
This leads to an asymptotic gain~$\nu_0^{-1} = d$ and~$\beta = d/2 - 6$ leads to fast convergence in~\eqref{eq:nud_elln} even when~$d$ is large.
\end{example}

\begin{example}[Local translations]
\label{ex:local_translations}
Let $d = s \cdot r$ (with $r,s \geq 5$ for simplicity), and consider the group composed of traslations over $r$ blocks of coordinates of size $s$; \ie, the block-cyclic group 
$$
G = \{ \sigma \;:\; \sigma = \sigma^{(1)}\circ\cdots \circ \sigma^{(r)} \}
$$
where each $\sigma^{(i)}$ is a translation over the set $\{(i-1)s + 1,\dots, is\}$, for $i\in[r]$.
Then it holds
\begin{equation}
\label{eq:local_translations}
\gamma_d(k) = \frac{1}{s^r} + O\left(k^{-s/2 + 1}\right)~.
\end{equation}
This follows by noticing that every local translation $\sigma$ (but the identity) satisfies $c(\sigma) \leq (d - s) + s/2$.
Here the asymptotic gain is~$\nu_0^{-1} = s^r = s^{d/s}$, which can be exponential in~$d$ when~$s$ is small. We have~$\beta = s/2 - 1$, which leads to much slower convergence than the translation case, unless~$s$ is large and of order~$d$.
\end{example}

\begin{example}[Full permutation group]\label{example:full_permutation_group}
For the case of $G = \Sym_d$, we can split the group based on the value of $\mathrm{Fix}(\sigma)$, the number of elements fixed by a permutation $\sigma$. Denote
$$
\xi(G,s) := |{\sigma \in G \;:\; \mathrm{Fix}(\sigma) > s}| =  \sum_{j = s+1}^d \binom{d}{j}\;!(d-j)~.
$$
for $s \in [d-1]$, where~$!k$ denotes the~$k$-th subfactorial. Then we have
$$
\gamma_d(k) \leq \frac{\xi(G,s)}{d!} + O\left( k^{-d/2 + \max\{ s/2,\; 6\}} \right)~,
$$
with equality for $s = d-1$. This follows from the fact that $c(\sigma) \leq \mathrm{Fix}(\sigma) + (d - \mathrm{Fix}(\sigma))/2$. In particular, it follows
$$
\gamma_d(k) \leq \frac{2}{(s+1)!} + O\left( k^{-d/2 + \max\{ s/2,\; 6\}} \right).
$$
When considering the full group, we may get a large asymptotic improvement of order~$\nu_0^{-1} = |G| = d!$ in sample complexity, but a slow convergence with~$\beta = -1$ as per Corollary~\ref{theo:number_of_cycles} (assuming~$d$ large enough).
Using different values of~$s$
may yield different upper bounds with faster convergence rates~$\beta = d/2 - \max\{s/2,6\}$, but smaller asymptotic gains in sample complexity, given by~$\nu_0^{-1} = (s + 1)! / 2$.
For instance, with~$s = d/2$ and~$d$ large enough, we have~$\beta = d/4$, leading to a potentially fast convergence rate in~\ref{eq:nud_elln} towards a sample complexity gain that is still significantly large, of order~$(d/2 + 1)! / 2$.
\end{example}

Overall, these examples show that the size of the group determines the best possible improvement in sample complexity, while the spectral properties of its permutations dictate how quickly we may achieve these gains.

\section{Beyond Group Invariance: Geometric Stability}
\label{sec:stability}

In this section, we study gains in sample complexity when the target function~$f^*$ is not fully invariant to a group, but may be stable under small geometric changes on the input.
We formalize this by considering a similar averaging operator~$S_G$, but we allow~$G$ to be a generic set of permutations instead of a group, and allow for a weighted average:
\begin{equation}
S_G f(x) := \sum_{\sigma \in G} h(\sigma) f(\sigma \cdot x),
\end{equation}
where~$h(\sigma) \geq 0$ for all~$\sigma \in G$ and~$\sum_{\sigma \in G} h(\sigma) = 1$.
We assume that~$G$ is ``symmetric'', \ie,~$\sigma^{-1} \in G$ when~$\sigma \in G$ and~$h(\sigma^{-1}) = h(\sigma)$, so that~$S_G$ is self-adjoint.
In this case, images of~$S_G$ are not invariant functions, but may nevertheless exhibit a form of ``local'' stability to small perturbations of the input data.
For instance, if~$G$ consists of local translations by at most a few pixels, or if~$G$ consists of all translations but~$h$ is localized around the identity, then applying $S_G$ yields functions that are stable to local translations.
We may also consider a more structured set~$G$ of permutations that resemble local deformations, consisting of both a global translation as well as different local translations at different scales, as we describe below.

\paragraph{Spectral properties of~$S_G$.}
Note that in this setup, we no longer have that~$S_G$ is a projection, however we may still view it as a \emph{smoothing} operator, which attenuates certain harmonics that are ``less'' invariant than others.
The next lemma shows related spectral properties to the invariant case. 
\begin{lemma}[Spectral properties of~$S_G$]
\label{lemma:spectral_stability}
There exists a basis of spherical harmonics~$\Ybar_{k,j}$, for $k \geq 0$, and $j = 1, \ldots, N(d,k)$, in which the operator~$S_G$ is diagonal, with eigenvalues~$\lambda_{k,j} \geq 0$.
In analogy to Lemma~\ref{lemma:projection}, we have
\begin{equation}
\label{eq:gammadk_stability}
\gamma_d(k) := N(d,k)^{-1} \sum_{j=1}^{N(d,k)} \lambda_{k,j} = \sum_{\sigma \in G} h(\sigma) \EE{x}{P_{d,k}(\langle \sigma\cdot x, x \rangle)}.
\end{equation}
We also define~$\nu_d(\ell) := \sup_{k \geq \ell} \gamma_d(k)$.
\end{lemma}

\paragraph{Sample complexity of stable kernel.}
In analogy to Section~\ref{sec:stat}, we may consider a stable kernel
\begin{equation}
\label{eq:stable_kernel}
\Kbar(x, x') = \sum_{\sigma \in G} h(\sigma) \kappa( \langle \sigma\cdot x, x' \rangle).
\end{equation}
Then, it is easy to check that the integral operator of~$\Kbar$ is given by~$T_\Kbar = S_G T_K$.
In contrast to Section~\ref{sec:stat}, we no longer have that the approximation errors of~$K$ and~$\Kbar$ are the same in general on ``geometrically stable" functions, since the notion is not precisely defined.
Nevertheless, we may represent favorable targets~$f^*$ as those whose coefficients decay similarly at each frequency~$k$ to those of~$S_G$, by viewing it as a smoothing of some~$L^2$ function~$g^*$, \ie, $f^* = S_G^r g^*$ for some exponent~$r$.
With this in mind, we make the following assumptions, replacing assumptions~(A1-3) of Section~\ref{sec:stat}.
\begin{itemize}[leftmargin=1cm]
	\item[(A5)] \emph{capacity}: the eigenvalues~$(\xi_m)_{m \geq 0}$ of~$T_K$ satisfy~$\xi_m \leq C(m+1)^{-\alpha}$.
	\item[(A6)] \emph{source condition}: there exists~$r > \frac{\alpha - 1}{2 \alpha}$ and~$g \in L^2(d \tau)$ with~$\|g\|_{L^2(d \tau)} \leq C_{f^*}$ such that~$f^* = S_G^{r} T_K^r g$.
\end{itemize}
Note that (A6) corresponds to a standard source condition with the kernel~$\Kbar$ (since~$T_\Kbar^r = S_G^{r} T_K^r$), yet it reveals how~$\Kbar$ jointly performs smoothing on the sphere, through~$T_K$, as well as on permutations through~$S_G$.
While these two forms of smoothing appear ``entangled'' in this assumption, one may balance them by choosing different levels of smoothing in the kernel function~$\kappa$, or by averaging multiple times in~\eqref{eq:stable_kernel}.
Assumption~(A5) is needed for obtaining a variant of Lemma~\ref{lemma:dof}, and implies~(A1) with~$C_K \propto C^{1/\alpha}$.
We then obtain the following generalization bound.

\begin{theorem}[Generalization with geometric stability.]
\label{thm:generalization_stability}
Assume (A4-6), and assume~$\nu_0 := \inf_{\ell \geq 0} \nu_d(\ell) > 0$.
Let $n \geq \max ( \|f^*\|_\infty^2 / \sigma_\rho^2, \left( C_1/\nu_0 \right)^{1/(2 \alpha r + 1 - \alpha)} )$,
and define
\begin{equation}
\label{eq:ln_equation_stability}
    \ell_n := \sup \{\ell : D(\ell) \leq C_2 \nu_d(\ell)^{\frac{2 r}{2 \alpha r + 1}} n^{\frac{1}{2 \alpha r + 1}}\}.
\end{equation}
We then have, for~$\lambda = C_3 (\nu_d(\ell_n)^{1/\alpha} / n)^{\alpha/(2 \alpha r + 1)}$,
\begin{equation}
\label{eq:invariant_bound}
\E[R(\hat f_\lambda) - R(f^*)] \leq C_4 \left(\frac{ \nu_d(\ell_n)^{1/\alpha}}{n}\right)^{\frac{2 \alpha r}{2 \alpha r + 1}}.
\end{equation}
In the same setting, KRR with kernel~$K$ achieves a similar bound with~$\nu_d(\ell_n)^{1/\alpha}$ replaced by 1, but with a possibly smaller constant~$C_4$.
Here, the constants~$C_{1:4}$ only depend on the parameters of assumptions (A4-6).
\end{theorem}

Note that the obtained generalization bound is very similar to Theorem~\ref{thm:generalization}, but with a factor~$\nu_d(\ell_n)^{1/\alpha}$ instead of~$\nu_d(\ell_n)$.
This is due to the fact that in contrast to the invariant case, where~$\gamma_d(k)$ in~\eqref{eq:gammadk} can help precisely control the number of invariant spherical harmonics, in this case~$\gamma_d(k)$ as computed in~\eqref{eq:gammadk_stability} can only give information about the sum of the eigenvalues~$\lambda_{k,j}$ at frequency~$k$, which may be insufficient to precisely estimate the gains in effective dimension.
The gap between these two factors is relatively small for kernels with slow decays ($\alpha \approx 1$) but can be more pronounced for smooth kernels with fast decays (large~$\alpha$).
Note also that the different source condition (A6) leads to a different approximation error, and thus to an approximation-estimation trade-off related to stability, which does not appear in the group-invariant case. More precise estimates of the decays of~$\lambda_{k,j}$ may help characterize this tradeoff more formally, and we leave this question to future work.
As in~\eqref{eq:nud_elln}, we may derive an estimate of~$\nu_d(\ell_n)$, namely if~$\nu_d(\ell) \approx \nu_0 + c \ell^{-\beta}$, then we have
\begin{equation}
\label{eq:nud_elln_stability}
    \nu_d(\ell_n) \leq \nu_0 + C \min\left\{ (\nu_0^{2 r} n)^{\frac{-\beta}{(d-1)(2 \alpha r + 1)}}, n^{\frac{-\beta}{(d-1)(2 \alpha r + 1) + 2 \beta r}}\right\}.
\end{equation}

\paragraph{Deformation-like stability.}
For inputs $x$ defined as signals $x \in L^2(\Omega)$ over a continuous domain $\Omega \subseteq \R^s$, $s=1,2$, the action of `small' diffeomorphisms $\varphi: \Omega \to \Omega$ as 
$(\varphi \cdot x)(u) = x( \varphi^{-1}(u))$ is a powerful diagnostic of performance of trainable CNNs \cite{petrini2021relative}, and a key guiding principle for scattering representations \cite{mallat2012group, bruna2013invariant}. In these works, the basic deformation cost is measured as $\| \varphi \| := \sup_u \| \nabla \varphi(u) - {\bf I} \|$. We instantiate an equivalent of small deformations in our finite-dimensional setting as follows.
\begin{equation}
\label{eq:defmodel}
\Phi_\varepsilon := \{ \sigma \in \Sym_d ~:~ | \sigma(u) - \sigma(u') - (u - u') | \leq \varepsilon |u - u'| \} ~,    
\end{equation}
where the differences are taken modulo $d$.
For $\varepsilon=0$, we recover the translation group described in Example \ref{ex:transl}. We can verify that $\varepsilon=1$ also corresponds to the translation group (due to the constraint that $\sigma(u) \neq \sigma(u')$ whenever $u\neq u'$), thus the first non-trivial model corresponds to $\varepsilon=2$. 
\begin{proposition}[Upper bound on~$\gamma_d(k)$ for deformations.]
\label{prop:defgroup}
It holds $\Phi_2^{-1} = \Phi_2$. Moreover, $|\Phi_2| \geq \tau^d$ for $\tau \approx 1.714$, and
\begin{equation}
    \gamma_d(k) \leq C \left(\frac{e^{2\eta}}{(2\eta)^{2\eta}\tau^{1-2\eta}} \right)^d + O \left(k^{-\eta d}\right)~
\end{equation}
for $\eta < 1/4$. In particular, $\tilde{\tau}^{-1}:=\frac{e^{2\eta}}{(2\eta)^{2\eta}\tau^{1-2\eta}} <1$ for $\eta < 0.07$, leading to an effective gain in sample complexity exponential in $d$, $\nu_0^{-1/\alpha} = \Theta(\tilde{\tau}^{d/\alpha})$; and $\beta = \eta d$ resulting in fast convergence in \eqref{eq:nud_elln_stability} even for large $d$.  
\end{proposition}
We thus verify that small deformations, already with $\varepsilon=2$, provide a substantial gain relative to rigid translations, since $\Phi_2$ now grows exponentially with the dimension, rather than linearly.
Let us remark that our small deformation model (\ref{eq:defmodel}) acting on $\{1,d\}$ 
differs in important ways from diffeomorphisms acting on a continuous domain. In our case they define unitary operators (since they are constructed as subsets of the permutation group), as opposed to diffeorphims, for which $\| \varphi \cdot x\|_{L^2(\Omega)} \neq \| x \|_{L^2(\Omega)}$ generally. In other words, the `deformations' in $\Phi_\varepsilon$ are more akin to local shufflings of the pixels rather than local distortions. That said, our model does roughly capture the size of small deformation classes. 
An interesting question for future work is to extend our framework to non-unitary transformations, which could accommodate appropriate discretisations of continuous diffeomorphisms.

\section{Numerical Experiments}
\label{sec:experiments}

\begin{figure}
    \centering
    \includegraphics[width=.32\textwidth]{"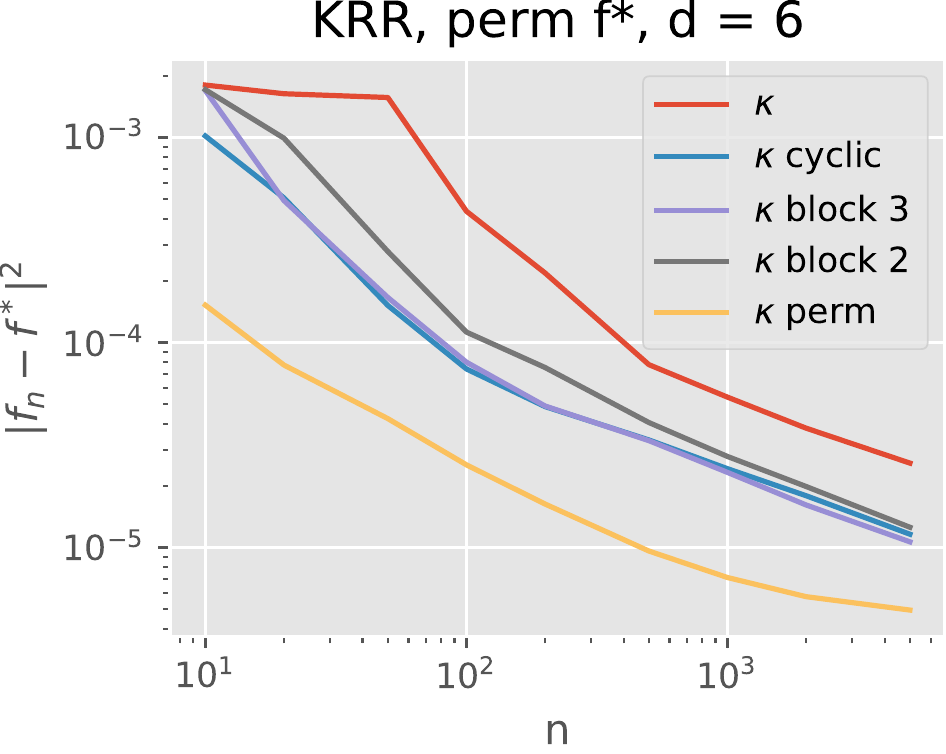"}
    \includegraphics[width=.32\textwidth]{"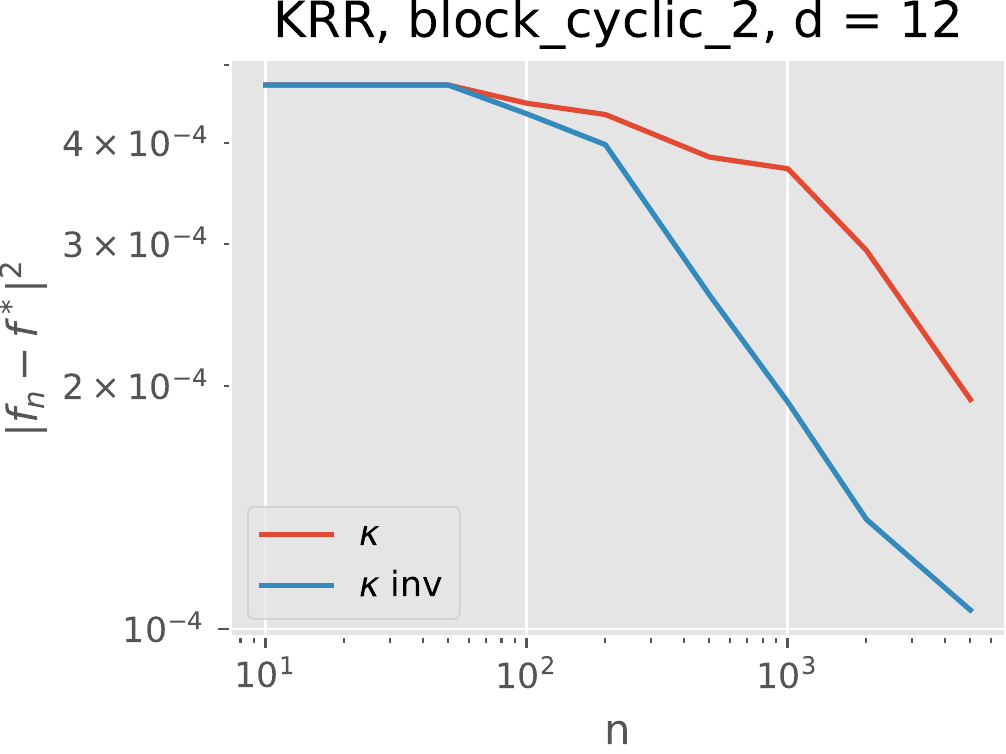"}
    \includegraphics[width=.32\textwidth]{"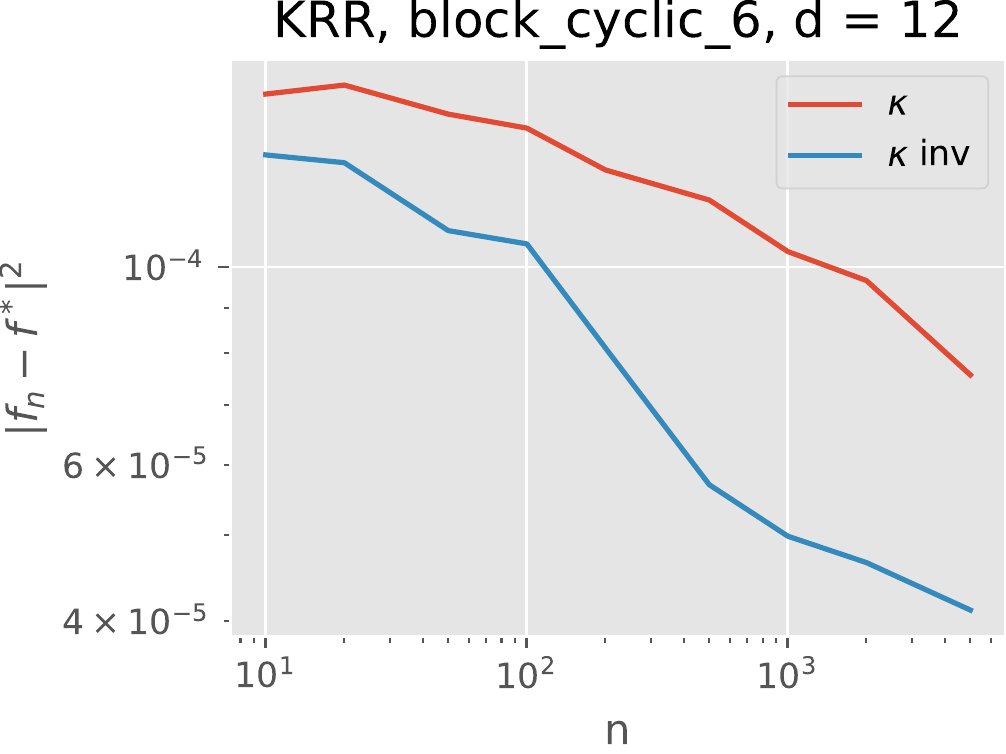"}

    \caption{Comparison of KRR with invariant and non-invariant kernels. (left) permutation-invariant target with~$d=6$, comparison between various invariant kernels (cyclic, block-cyclic, and permutation groups). (center/right) invariant vs non-invariant kernels on invariant target functions with~$d=12$, for block-cyclic groups~$G$ of two different sizes.}
    \label{fig:krr_curves}
\end{figure}

In this section, we provide simple numerical experiments on synthetic data which illustrate our theoretical results.
In Figure~\ref{fig:krr_curves}, we consider KRR on 5\,000 training samples with inputs uniformly distributed on~$\Sbb^{d-1}$, and outputs generated according to a target non-smooth function~$f^* = S_G g^*$, with~$g^*(x) = \1\{w_*^\top x \geq 0.7\}$, where the averaging operator~$S_G$ is over different groups in each plot.
The regularization parameter~$\lambda$ is optimized on~5\,000 test samples.
We use the dot-product kernel function~$\kappa(u) = (u + 1) \kappa_1(u)$, where~$\kappa_1$ is the arc-cosine kernel of degree 1, which corresponds to an infinite-width shallow ReLU network~\cite{cho2009kernel}.

When the target is permutation-invariant, we can see in Figure~\ref{fig:krr_curves}(left) that the kernel based on permutation invariance leads to the largest gain in sample complexity compared to those which use cyclic or block-cyclic groups.
Since the permutation group has the largest cardinality, this is consistent with our finding that the gains may be of the order of the size of the group.
Figures~\ref{fig:krr_curves}(center/right) consider the example of cyclic translations on local blocks of size~$2$ or~$6$ (Example~\ref{ex:local_translations} with~$s=2$ or~$6$), and illustrate that the improvement in sample complexity happens later for~$s=2$ than~$s=6$, which is consistent with the slower decays of~$\gamma_d$ obtained in~\eqref{eq:local_translations} due to the larger number of cycles.

\section{Discussion and Conclusion}
\label{sec:discussion}

We have studied how geometric invariance or stability assumptions on target functions enable more efficient learning, with improvements in sample complexity which may be as large as the number of permutations considered in the group or set of elements to which the target is invariant or stable.
In particular, this gain can be exponential in the dimension if we consider, \eg, all permutations, local translations on small blocks, or permutations that resemble small deformations.
This last example provides a strong justification for seeking models and architectures that are stable to deformations, a natural prior when learning functions on images~\cite{petrini2021relative}.
In that respect, our results provide a theoretical baseline to assess learning guarantees under geometric priors: by designing appropriate geometrically stable kernels, we simultaneously address approximation and generalisation errors within a framework of convex optimization. 

That said, while these gains may be large in practice, the obtained rates are still generically cursed by dimension if the target is non-smooth. In other words, invariance or geoemtric stability allows us to express Lipschitz assumptions with respect to weaker metrics. While these stronger regularity assumptions result in important gains in sample complexity, they do not overcome the inherent difficulty of learning non-smooth structures in high-dimensions. 
This suggests that further assumptions may be needed to learn efficiently on high-dimensional geometric data, for instance with more structured forms of regularity beyond our invariant/stable setup, which may be exploited perhaps through architectures that involve local connectivity, hierarchy (which would explain the benefits of depth, as opposed to our current results), or feature learning~\cite{bach2017breaking,bietti2021approximation,favero2021locality,malach2020computational,poggio2017and}.
A natural further question is to study stability to transformations that are not necessarily permutations, and which may then provide more realistic models of continuous deformations.
Another interesting question is to study whether it is possible to adapt to general symmetries present in the target, instead of encoding them in the model with an appropriately designed kernel as done~here.

\begin{ack}
LV and JB acknowledge partial support from the Alfred P. Sloan Foundation, NSF RI-1816753, NSF CAREER CIF 1845360, NSF CHS-1901091, and Samsung Electronics.

\end{ack}

\bibliography{full,bibli}
\bibliographystyle{abbrv}

\newpage
\appendix

This appendix contains additional background on spherical harmonics (Appendix~\ref{sec:appx_background}), and proofs of the results from Section~\ref{sec:stat},~\ref{sec:approx} and~\ref{sec:stability} (in Appendix~\ref{sec:appx_stat},~\ref{sec:appx_decay} and~\ref{sec:appx_stability}, respectively).

\section{Background on Spherical Harmonics and Legendre Polynomials}
\label{sec:appx_background}

In this section, we provide background on spherical harmonic and Legendre/Gegenbauer polynomials, which are used extensively in our analysis.
See, \eg,~\cite{atkinson2012spherical,costas2014spherical} for references.
We consider inputs on the $d-1$ sphere~$\mathbb S^{\dmone} = \{x \in \R^d, \|x\| = 1\}$.

We recall some properties of the spherical harmonics~$Y_{k,j}$ introduced in Section~\ref{sec:background}.
For~$j = 1, \ldots, N(d, k)$, where~$N(d,k) = \frac{2k + d - 2}{k} {k + d - 3 \choose d - 2}$, the spherical harmonics~$Y_{k,j}$ are homogeneous harmonic polynomials of degree~$k$ that are orthonormal with respect to the uniform distribution~$\tau$ on the~$\dmone$ sphere.
The degree~$k$ plays the role of an integer frequency, as in Fourier series, and the collection~$\{Y_{k,j}, k \geq 0, j = 1, \ldots, N(d,k)\}$ forms an orthonormal basis of~$L^2(\Sbb^{\dmone}, d\tau)$.
As with Fourier series, there are tight connections between decay of coefficients in this basis w.r.t.~$k$, and regularity/differentiability of functions, in this case differentiability on the sphere.
In particular, this is a key property that we exploit for obtaining decays related to the number of invariant polynomials in Proposition~\ref{prop:decay}.
This follows from the fact that spherical harmonics are eigenfunctions of the Laplace-Beltrami operator on the sphere~$\Delta_{\Sbb^{d-1}}$ (see~\cite[Proposition 4.5]{costas2014spherical}):
\begin{equation}
\label{eq:laplace_beltrami}
\Delta_{\Sbb^{d-1}} Y_{k,j} = -k(k+d-2)Y_{k,j}.
\end{equation}

For a given frequency~$k$, we have the following addition formula:
\begin{equation}
\label{eq:spherical_addition}
\sum_{j=1}^{N(d, k)} Y_{k,j}(x) Y_{k,j}(y) = N(d, k) P_{d,k}( x^\top y ),
\end{equation}
where~$P_{d,k}$ is the $k$-th Legendre polynomial in dimension~$d$ (also known as Gegenbauer polynomial when using a different scaling),
given by the Rodrigues formula:
\begin{equation}
\label{eq:rodrigues}
P_{d,k}(t) = (-1/2)^k \frac{\Gamma(\frac{d-1}{2})}{\Gamma(k + \frac{d-1}{2})} (1 - t^2)^{(3-d)/2}
	\left(\frac{d}{dt}\right)^k (1 - t^2)^{k+(d-3)/2}.
\end{equation}

The polynomials~$P_{d,k}$ are orthogonal in~$L^2([-1, 1], dw)$ where the measure $d w$ is given by the weight function
$d w(t) = (1 - t^2)^{(d-3)/2}dt$, and we have
\begin{equation}
\label{eq:legendre_norm}
\int_{-1}^1 P_{d,k}^2(t) (1 - t^2)^{(d-3)/2}dt = \frac{\omega_{d-1}}{\omega_{d-2}} \frac{1}{N(d,k)},
\end{equation}
where~$\omega_{p-1} = \frac{2 \pi^{p/2}}{\Gamma(p/2)}$ denotes the surface of the sphere~$\mathbb S^{p-1}$ in~$p$ dimensions.
Using the addition formula~\eqref{eq:spherical_addition} and orthogonality of spherical harmonics, we can show
\begin{equation}
\label{eq:legendre_dp}
\int P_{d,j}( w^\top x ) P_{d,k}( w^\top y ) d \tau(w) = \frac{\delta_{jk}}{N(d,k)} P_{d,k}( x^\top y )
\end{equation}
We will use the following recurrence relation of Legendre polynomials~\cite[Eq. 4.36]{costas2014spherical}
\begin{equation}
\label{eq:legendre_rec}
t P_{d,k}(t) = \frac{k}{2k + d - 2} P_{d,\kmone}(t) + \frac{k + d - 2}{2k + d - 2} P_{d,k+1}(t),
\end{equation}
for $k \geq 1$, and for $k = 0$ we simply have $t P_{d,0}(t) = P_{d,1}(t)$.
We will also use the following pointwise upper bound on~$P_{d,k}(t)$ (see~\cite[Eq.~2.117]{atkinson2012spherical}):
\begin{equation}
\left| P_{d,k}(t) \right|  \leq \frac{\Gamma\left(\frac{d-1}{2}\right)}{\sqrt{\pi}} \left(\frac{4}{k(1-t^2)}\right)^{(d-2)/2}
\end{equation}

The Funk-Hecke formula is helpful for computing Fourier coefficients in the basis of spherical harmonics in terms of
Legendre polynomials: for any~$j = 1, \ldots, N(d, k)$, we have
\begin{equation}
\label{eq:funk_hecke}
\int f(x^\top y) Y_{k,j}(y) d \tau(y) = \frac{\omega_{d-2}}{\omega_{d-1}} Y_{k,j}(x) \int_{-1}^1 f(t) P_{d,k}(t) (1 - t^2)^{(d-3)/2} dt.
\end{equation}
For example, we may use this to obtain decompositions of dot-product kernels by computing Fourier coefficients of functions~$\kappa(\langle x, \cdot \rangle)$.
Indeed, denoting
\begin{equation*}
\mu_k = \frac{\omega_{d-2}}{\omega_{d-1}} \int_{-1}^1 \kappa(t) P_{d,k}(t) (1 - t^2)^{(d-3)/2}dt,
\end{equation*}
writing the decomposition of~$\kappa(\langle x, \cdot \rangle)$ using~\eqref{eq:funk_hecke} leads to the following Mercer decomposition of the kernel:
\begin{equation}
\label{eq:mercer}
\kappa(x^\top y) = \sum_{k=0}^\infty \mu_k \sum_{j=1}^{N(d,k)} Y_{k,j}(x) Y_{k,j}(y) = \sum_{k=0}^\infty \mu_k N(d,k) P_{d,k}(x^\top y).
\end{equation}

\section{Proofs for Section~\ref{sec:stat} (Sample complexity of Invariant Kernels)}
\label{sec:appx_stat}

\subsection{Proof of Lemma~\ref{lemma:eigenvalues} (spectral properties of~$K$ and~$\Kbar$)}

\begin{proof}
That~$\mu_k$ are eigenvalues of~$T_K$ with multiplicity~$N(d,k)$ is standard and follows from the Funk-Hecke formula (see, \eg,~\cite{smola2001regularization}).
In particular, for any spherical harmonic~$Y_k \in V_{d,k}$, we have~$T_K Y_k = \mu_k Y_k$, and there are~$N(d,k)$ orthogonal spherical harmonics in~$V_{d,k}$.

For~$\Kbar$, note that we have~$T_\Kbar = S_G T_K$, so that for any $G$-invariant spherical harmonic~$\Ybar_{k} \in \Vbar_{d,k}$ we have $T_\Kbar \Ybar_{k} = \mu_k S_G \Ybar_k = \mu_k \Ybar_k$, while for any~$Y_k \in V_{d,k} \cap \Vbar_{d,k}^{\perp}$, we have~$T_\Kbar Y_k = 0$ since~$S_G Y_k = 0$.
\end{proof}

\subsection{Proof of Lemma~\ref{lemma:approx_error} (approximation error)}

\begin{proof}
Let~$\Ybar_{k,j}$, for~$j = 1, \ldots, N(d,k)$ be an orthonormal basis of~$V_{d,k}$ such that~$(\Ybar_{k,j})_{j \leq \Nbar(d,k)}$ form an orthonormal basis of~$\Vbar_{d,k}$.
Then, the collection of~$\Ybar_{k,j}$ for~$k \geq 0$ and~$j = 1, \ldots, N(d,k)$ forms an orthonormal basis of~$L^2(d \tau)$.

For a function~$f \in L^2(d \tau)$ with decomposition
\begin{equation*}
f(x) = \sum_{k \geq 0} \sum_{j=1}^{N(d,k)} a_{k,j} \Ybar_{k,j}(x),
\end{equation*}
we have the following expressions of the RKHS norms for~$K$ and~$\Kbar$ by Mercer's theorem (\eg,~\cite{cucker2002mathematical}):
\begin{align*}
\|f\|^2_{\Hcal_K} &= \begin{cases}
	\sum_{k : \mu_k > 0} \sum_{j=1}^{N(d,k)} \frac{a_{k,j}^2}{\mu_k}, &\text{ if }a_{k,j} = 0\text{ whenever }\mu_k = 0\\
	\infty, &\text{ otherwise.}
	\end{cases} \\
\|f\|^2_{\Hcal_\Kbar} &= \begin{cases}
	\sum_{k : \mu_k > 0} \sum_{j=1}^{\Nbar(d,k)} \frac{a_{k,j}^2}{\mu_k}, &\text{ if }a_{k,j} = 0\text{ whenever }\mu_k = 0\text{ or }j > \Nbar(d,k)\\
	\infty, &\text{ otherwise.}
\end{cases}
\end{align*}

Assume now that~$f^*$ is invariant, so that its coefficients~$a^*_{k,j}$ satisfy~$a^*_{k,j} = 0$ for~$j > \Nbar(d,k)$.
We have
\begin{align*}
A_{\Hcal_K}(\lambda, f^*) &= \inf_{f \in \Hcal_K} \|f - f^*\|_{L^2(d \tau)}^2 + \lambda \|f\|_{\Hcal_K}^2 \\
	&= \inf_{a_{k,j}} \sum_{k : \mu_k > 0} \sum_{j=0}^{N(d,k)} \left((a_{k,j} - a^*_{k,j})^2 + \lambda \frac{a_{k,j}^2}{\mu_k} \right)\\
	&= \inf_{a_{k,j}} \sum_{k : \mu_k > 0} \sum_{j=0}^{\Nbar(d,k)} \left((a_{k,j} - a^*_{k,j})^2 + \lambda \frac{a_{k,j}^2}{\mu_k} \right) + \sum_{k : \mu_k > 0} \sum_{j=\Nbar(d,k)+1}^{N(d,k)} (1 + \lambda/\mu_k) a_{k,j}^2 \\
	&= \inf_{a_{k,j}} \sum_{k : \mu_k > 0} \sum_{j=0}^{\Nbar(d,k)} \left((a_{k,j} - a^*_{k,j})^2 + \lambda \frac{a_{k,j}^2}{\mu_k} \right) \\
	&= A_{\Hcal_{\Kbar}}(\lambda, f^*),
\end{align*}
which proves the lemma.
\end{proof}

\subsection{Proof of Lemma~\ref{lemma:dof} (degrees of freedom)}

\begin{proof}
The result immediately follows from the following expressions of degrees of freedom for~$K$ and~$\Kbar$:
\begin{equation}
\label{eq:dof}
\Ncal_K(\lambda) = \sum_{k \geq 0} N(d,k) \frac{\mu_k}{\mu_k + \lambda}, \qquad \Ncal_{\Kbar}(\lambda) = \sum_{k \geq 0} \Nbar(d,k)\frac{\mu_k}{\mu_k + \lambda}.
\end{equation}
\end{proof}

\subsection{Proof of Theorem~\ref{thm:generalization} (generalization bound)}

\begin{proof}
We start from the following bound from~\cite[Proposition 7.2]{bach2021ltfp}, which holds for any~$\lambda \leq 1$, assuming~$\Kbar(x,x) \leq 1$ almost surely (this is satisfied when~$\kappa(u) \leq 1$), and for~$n \geq \frac{5}{\lambda}(1 + \log(1/\lambda))$:

\begin{equation}
\label{eq:bach_bound}
\E [R(\hat f_\lambda)] - R(f^*) \leq 16 \frac{\sigma_q^2}{n} \Ncal_\Kbar(\lambda) + 16 A_{\Hcal_{\Kbar}}(\lambda, f^*) +  \frac{24}{n^2}\|f^*\|^2_\infty.
\end{equation}

Under assumption (A2), we have (see, \eg, \cite[Theorem 3, p.33]{cucker2002mathematical}, using that~$\|f\|_{\Hcal_K} = \|T_K^{-1/2} f\|_{L^2(d \tau)}$)
\begin{equation}
A_{\Hcal_K}(\lambda, f^*) \leq C_{f^*}^2 \lambda^{2r},
\end{equation}
with~$C_{f^*} := \|T_K^{-r} f^*\|_{L^2(d \tau)}$
By Lemma~\ref{lemma:approx_error}, we also have
\begin{equation}
A_{\Hcal_\Kbar}(\lambda, f^*) \leq C_{f^*}^2 \lambda^{2r}.
\end{equation}
Using Lemma~\ref{lemma:dof} for some integer~$\ell \geq 0$ and (A1), the bound~\eqref{eq:bach_bound} becomes
\begin{equation}
\label{eq:krr_generic_bound}
\E [R(\hat f_\lambda)] - R(f^*) \leq 16 C_{f^*}^2 \lambda^{2r} + 16 \frac{\sigma_q^2 D(\ell)}{n} + 16 \frac{C_K \sigma_q^2 \nu_d(\ell)}{n} \lambda^{-1/\alpha} +  \frac{24}{n^2}\|f^*\|^2_\infty.
\end{equation}

Jointly optimizing the first and third terms for~$\lambda$ yields
\begin{equation}
\lambda_n = \left( \frac{C_K \sigma_q^2 \nu_d(\ell)}{2 r \alpha C_{f^*}^2 n} \right)^{\frac{\alpha}{2 \alpha r + 1}}
\end{equation}
The bound then becomes
\begin{equation}
\E [R(\hat f_{\lambda_n})] - R(f^*) \lesssim C_{f^*}^{\frac{2}{2 \alpha r + 1}} \left( \frac{C_K \sigma_q^2 \nu_d(\ell)}{n}\right)^{\frac{2 \alpha r}{2 \alpha r + 1}} + \frac{\sigma_q^2 D(\ell)}{n} +  \frac{1}{n^2}\|f^*\|^2_\infty.
\end{equation}
Here, $\lesssim$ hides only absolute constants that depend on~$\alpha$ and~$r$.
Now we choose~$\ell = \ell_n$, given by~\eqref{eq:ln_equation}, with constant corresponding to:
\begin{equation}
    \ell_n := \sup \{\ell : \sigma_q^2 D(\ell) \leq C_{f^*}^{\frac{2}{2 \alpha r + 1}} \left( C_K \sigma_q^2 \nu_d(\ell)\right)^{\frac{2 \alpha r}{2 \alpha r + 1}} n^{\frac{1}{2 \alpha r + 1}} \},
\end{equation}
so that the second term is smaller than the first term.
The last term is of the same order when
\begin{equation}
n \gtrsim \frac{\|f^*\|_\infty^2}{\sigma_q^2 D(\ell_n)},
\end{equation}
which is verified under the condition
\begin{equation}
\label{eq:n_lowerbound}
n \geq \max \left\{ \|f^*\|_\infty^2 / \sigma_\rho^2, \left( C_1/\nu_0 \right)^{\frac{\alpha}{2 \alpha r + 1 - \alpha}} \right\}
\end{equation}
from the theorem statement, since~$D(\ell_n) \geq 1$.
Note that for the specific bound~\eqref{eq:bach_bound} to hold we also need~$\lambda_n \gtrsim 1/n$ (up to logarithmic terms).
This imposes the qualification condition~$r > (\alpha - 1) / 2 \alpha$, and leads to the additional requirement
\begin{equation}
n \gtrsim \left( \frac{C_{f^*}^2}{\sigma_q^2 C_K \nu_d(\ell)} \right)^{\frac{\alpha}{2 \alpha r + 1 - \alpha}}.
\end{equation}
This is verified under condition~\eqref{eq:n_lowerbound} with~$C_1 = C_{f^*}^2 / \sigma_q^2 C_K$.

For the KRR estimator with kernel~$K$, the same bound~\eqref{eq:krr_generic_bound} holds, but without the factor~$\nu_d(\ell)$ and without the term involving~$D(\ell)$. The resulting bound follows from a similar analysis.
\end{proof}

\subsection{Estimating~$\ell_n$ and the effective gain~$\nu_d(\ell_n)$}
\label{sub:appx_ell_n}

In this section, we provide more details on our study of the asymptotic behavior of the quantities~$\ell_n$ and~$\nu_d(\ell_n)$ in Theorem~\ref{thm:generalization}, as described in Section~\ref{sec:stat}.

Since~$D(\ell)$ increases with~$\ell$, Eq.~\eqref{eq:ln_equation} suggests that~$\ell_n$ increases with~$n$.
We now provide intuition on how we might expect~$\ell_n$ and~$\nu_d(\ell_n)$ to behave in a situation of interest where we know an asymptotic equivalent of~$\nu_d$.
Namely, assume that
\begin{equation*}
    \nu_d(\ell) \approx \nu_0 + c \ell^{-\beta}.
\end{equation*}
We provide such asymptotic equivalents in Section~\ref{sec:approx}, where~$\nu_0 = 1/|G|$, and~$\beta$ depends on spectral properties of the elements of~$G$.
For some large groups,~$\beta$ may be small, in which case we may consider other approximations with larger~$\beta$, at the cost of a larger~$\nu_0$.
When~$\ell_n$ is large, using the approximation~$N(d,k) \approx k^{d-2}$, we have~$D(\ell) \approx \sum_{k=0}^{\ell-1} k^{d-2} \approx \ell^{d-1}$.
Hiding constants other than~$\nu_0$, we may then consider~$\ell_n$ to be solution of
\[
\ell^{\frac{(d-1)(2 \alpha r + 1)}{2 \alpha r}} = n^{\frac{1}{2\alpha r}} (\nu_0 + \ell^{-\beta}).
\]
Since the l.h.s.~increases, while the r.h.s.~decreases with~$\ell$, we must have~$\ell_n \geq \max(\ell_{n,1}, \ell_{n,2})$,
with
\begin{equation*}
    \ell_{n,1}^{\frac{(d-1)(2 \alpha r + 1)}{2 \alpha r}} = n^{\frac{1}{2\alpha r}} \nu_0, \quad \text{ and } \quad \ell_{n,2}^{\frac{(d-1)(2 \alpha r + 1)}{2 \alpha r}} = n^{\frac{1}{2\alpha r}} \ell_{n,2}^{-\beta}.
\end{equation*}
This yields
\begin{equation}
    \nu_d(\ell_n) \leq \nu_0 + C \min\left\{ (\nu_0^{2 \alpha r} n)^{\frac{-\beta}{(d-1)(2 \alpha r + 1)}}, n^{\frac{-\beta}{(d-1)(2 \alpha r + 1) + 2 \beta \alpha r}}\right\}.
\end{equation}
Notice that when~$\beta \ll d$, both exponents of~$n$ display a curse of dimensionality, but this curse goes away as~$\beta$ grows.
Note also that the first exponent yields a faster rate, but one that is only achieved for large~$n$ due to the factor~$\nu_0^{2 \alpha r}$, which may be small for large groups.

\subsection{Discussion of Optimality}
\label{sub:appx_optimality}
In this section, we discuss the optimality of the upper bounds in Theorem~\ref{thm:generalization}, in particular the constant~$C_4$ and its dependence on the constants~$C_K$ and~$C_{f^*}$ from the source and capacity conditions.

\paragraph{Tightness of~$C_4$ for non-invariant targets.}
We first provide a minimax lower bound for the class of (non-invariant) targets satisfying the source and capacity conditions (A1/A2), in order to show that the constant~$C_4$ can be tight (up to absolute constants) in a minimax sense for this class.

Consider a kernel $K_0(x, x') = \kappa_0(\langle x, x' \rangle)$ such that we have the following asymptotics on the eigenvalues of the integral operator: $\lambda_m(T_{K_0}) \sim C_0 m^{-\alpha}$.
Let~$\ell_0$ be such that for all~$m \geq M_0 := D(\ell_0) + 1$ we have~$C_0 m^{-\alpha} / 2 \leq \lambda_m(T_{K_0}) \leq 2 C_0 m^{-\alpha}$.
We can then construct a function~$\kappa$ and corresponding kernel~$K$ such that~$\lambda_m(T_K) \leq 2C_0 m^{-\alpha}$ and for~$m \geq M_0$, $\lambda_m(T_K) \geq C_0 m^{-\alpha} / 2$.
For instance, we may define~$\kappa(u) = \sum_{k \geq 0}  \mu_k(\kappa) N(d,k) P_{d,k}(u)$, with
\[
\mu_k(\kappa) = \begin{cases} 2 C_0 M_0^{-\alpha} &\text{ if }k \leq \ell_0\\
    \mu_k(\kappa_0) &\text{otherwise,}
    \end{cases}
\]
where the~$\mu_k(\kappa_0)$ are the Legendre coefficients of~$\kappa_0$.

Note that this kernel~$K$ satisfies the capacity condition (A1) with~$C_K \lesssim C_0^{1/\alpha}$. With this choice of~$K$, define~$\mathcal F$ to be the set of regression functions~$f^*$ that further satisfy assumption (A2) with parameters~$C_*$ and~$r$, and assume that labels are generated as~$y = f^*(x) + \epsilon$, with~$\epsilon \sim \mathcal N(0, \sigma_\rho^2)$.
Under these assumptions, note that the upper bound in Theorem~\ref{thm:generalization} is given by
\begin{equation}
\label{eq:upper_bound_constants}
\E[\|\hat f - f^*\|^2] \lesssim C_*^{\frac{2}{2 \alpha r + 1}} C_0 ^{\frac{2 r}{2 \alpha r + 1}} (\frac{\sigma_\rho^2}{n})^{\frac{2 \alpha r}{2 \alpha r + 1}},
\end{equation}
where~$\lesssim$ hides absolute constants or constants depending only on~$\alpha$ and~$r$.

Following~\cite{bach2021ltfp}, we use Fano's inequality to lower bound the minimax risk. In particular, we have a lower bound
\[
M_n(\mathcal F) := \inf_{\hat f} \sup_{f^* \in \mathcal F} \E_{\mathcal D_n \sim \rho^{\otimes n}}[\|\hat f_{\mathcal D_n} - f^*\|_{L^2(d \tau)}^2] \geq A / 2,
\]
on the minimax risk~$M_n(\mathcal F)$ if we can find a set $\{f_1, \ldots, f_M\} \in \mathcal F$, $M \geq 16$, such that
\begin{itemize}
    \item $\|f_i - f_j\|_{L^2(d \tau)}^2 \geq 4 A$ for~$i \ne j$ (\ie, we have a packing set)
    \item $\frac{n}{2 \sigma_\rho^2} \|f_i - f_j\|_{L^2(d \tau)}^2 \leq \frac{\log M}{4}$ (this ensures~$f_i$ and~$f_j$ are difficult enough to distinguish).
\end{itemize}

In order to construct a packing, we use the Varshamov-Gilbert lemma to obtain~$M \geq \exp(K / 8)$ elements~$x_1, \ldots, x_M \in \{0, 1\}^K$ for some~$K$ to be chosen later, which satisfy $\|x_i - x_j\|_1 \geq K/4$ for~$i \ne j$.
Defining~$f_i = \beta \sum_{m=1}^K 2((x_i)_m - 1) \phi_m$, where~$(\phi_m)_m$ are the eigenfunctions of~$T_K$ sorted such that the corresponding eigenvalues~$\lambda_m$ are non-decreasing, we have
\[
\|f_i - f_j\|^2 \geq \beta^2 K, \quad \text{for } i \ne j,
\]
and may thus consider a lower bound of the form~$A/2 = K \beta^2 / 8$.
Then, since~$\|f_i - f_j\|^2 \leq 4 \beta^2 K \leq 32 \beta^2 \log M$, it suffices to have~$16 n \beta^2 / \sigma_\rho^2 \log M \leq \log M /4$, \ie, $\beta^2 \leq \sigma_\rho^2 / 64 n$ to satisfy the second condition above.
Further, in order for all~$f_i$ to satisfy the capacity condition, we need~$\|T_K^{-r} f_i\|^2 \leq C_*^2$. Note that we have
\[
\|T_K^{-r} f_i\|^2 = \beta^2\sum_{m=1}^K \lambda_m^{-2r} \leq K \beta^2 \lambda_K^{-2r},
\]
thus, it suffices to take~$K \beta^2 \leq C_*^2 \lambda_K^{2r}$.
Taking a maximal~$\beta^2$ under these two conditions, we have the following lower bound on the minimax risk:
\[
M_n(\mathcal F) \geq \frac{K \beta^2}{8} \geq \frac{1}{8}\min\left\{ C_*^2 \lambda_K^{2r}, \frac{K \sigma_\rho^2}{64 n}\right\}.
\]
Using the lower bound~$\lambda_K^{2 r} \geq (C_0/2)^{2r} K^{-2 \alpha r}$, which holds for~$K \geq M_0$, and optimizing for~$K$ yields $K \approx (C_*^2 C_K^{2r} \sigma_\rho^2 n)^{1/(1 + 2 \alpha r)}$. For~$n$ large enough, we have~$K \geq M_0$, and the following lower bound holds:
\[
M_n(\mathcal F) \gtrsim C_*^{\frac{2}{2 \alpha r + 1}} C_0^{\frac{2 r}{2 \alpha r + 1}} (\frac{\sigma_\rho^2}{n})^{\frac{2 \alpha r}{2 \alpha r + 1}}.
\]
This matches the upper bound~\eqref{eq:upper_bound_constants} up to absolute constants.

\paragraph{Tightness of~$1 / |G|$ for the invariant class.}
We now show that the our bound in Theorem~\ref{thm:generalization} which asymptotically shows a $1/|G|n$ instead of~$1/n$, is (asymptotically) minimax optimal over the class of \emph{invariant} targets which satisfy assumption~(A2).

We consider the same kernel~$K_0$ as above, and denote by~$K_{G,0}$ its invariant counterpart. It suffices to show that we have the asymptotic expansion
\begin{equation}
\label{eq:inv_lambda_expansion}
\lambda_m(T_{K_{G,0}}) \sim |G|^{-\alpha} C_0 m^{-\alpha}
\end{equation}
instead of~$C_0 m^{-\alpha}$.
Indeed, in this case we can construct a kernel~$K$ such that its invariant counterpart~$K_G$ has eigenvalues upper bounded as $\lambda_m(T_{K_G}) \leq 2 C_0 |G|^{-\alpha} m^{-\alpha}$ and lower bounded by~$(C_0/2) |G|^{-\alpha} m^{-\alpha}$ for~$m \geq M_1$ (note that~$M_1$ could be chosen large enough so that construction from before for the non-invariant case also applies to the same kernel~$K$).
Then, applying the same arguments as for the non-invariant case, we obtain the desired minimax-lower bound
\[
M_n(\bar {\mathcal F}) \gtrsim C_*^{\frac{2}{2 \alpha r + 1}} C_0^{\frac{2 r}{2 \alpha r + 1}} (\frac{\sigma_\rho^2}{|G| n})^{\frac{2 \alpha r}{2 \alpha r + 1}},
\]
for~$n$ large enough, where~$\bar{\mathcal F}$ is the class of invariant targets satisfying assumption~(A2) with the kernel~$K$. This shows that our upper bound is asymptotically tight in a minimax sense for the kernel considered.

We now explain why~\eqref{eq:inv_lambda_expansion} holds.
Let~$\mu_k$ denote the Legendre coefficients of~$\kappa$ at frequency~$k$, and assume~$\mu_k \sim C_1 k^{-\beta}$.
Recall that we have $N(d,k) \sim C_2 k^{d-2}$, so that~$D(k) = \sum_{k' \leq k} N(d,k') \sim C_3 k^{d-1}$ for some~$C_3$.
Then, when taking eigenvalues with their multiplicity, when~$D(k-1) < m \leq D(k)$, the $m$-th eigenvalue~$\lambda_m(T_{K_0})$ is~$\mu_k$. Asymptotically, we have~$k \sim C_3^{\frac{1}{d-1}} m^{\frac{1}{d-1}}$, so that~$\lambda_m \sim C_1 C_3^{\frac{-\beta}{d-1}} m^{\frac{-\beta}{d-1}}$, that is, we have~$\alpha = \beta/(d-1)$ and~$C_0 = C_1 C_3^{-\alpha}$.

Now, since~$\frac{\Nbar(d,k)}{N(d,k)} \to \frac{1}{|G|}$ as $k \to \infty$, we have~$\Nbar(d,k) \sim (C_2/|G|) k^{d-2}$, and~$\bar D(k) \sim (C_3/|G|) k^{d-1}$. With the same reasoning, this leads to~$\lambda_m(T_{K_{G,0}}) \sim C_1 (C_3 / |G|)^{-\alpha} m^{-\alpha} = C_0 |G|^{-\alpha} m^{-\alpha}$, which is the desired constant.

\section{Proofs for Section \ref{sec:approx} (Decays of~$\gamma_d(k)$)}
\label{sec:appx_decay}

\subsection{Proof of Proposition~\ref{prop:decay} (decay of~$\gamma_{d,\sigma}(k)$)}

The proof of Proposition~\ref{prop:decay} is technical and relies on identifying and analyzing the singularities in the density~$q_\sigma$ of the random variable~$Z_\sigma = \langle \sigma \cdot x, x \rangle$, with~$x \sim \tau$, using results in~\cite{saldanha1996accumulated}.
Lemma~\ref{lemma:IbP} provides a general integration by parts result which is useful throughout the proof to obtain asymptotic decays from regularity properties.
Lemma~\ref{lemma:phi_interior} and Lemma~\ref{lemma:phi_pm1} provide asymptotic decays for singularities~$\phi(t)$ localized around some~$\lambda$ in~$(-1, 1)$ and~$\{\pm 1\}$, respectively, either through integration by parts or using closed form expressions of certain integrals.
Proposition~\ref{prop:decay} is then proved by appropriately ``cancelling'' the singularities in~$q_\sigma$ using such localized functions~$\phi$, as explained in Lemma~\ref{lemma:singularities}, and applying the integration by parts lemma on the resulting function, which is of higher smoothness and thus leads to faster-decaying terms.

\begin{lemma}[Integration by parts]\label{lemma:IbP}
Let~$g:[-1, 1] \to \R$ be~$2s$-times differentiable, with all derivatives bounded on~$[-1, 1]$. We then have
\begin{equation}
\label{eq:ibp}
    \int_{-1}^1 g(t) P_{d,k}(t) (1 - t^2)^{\frac{d-3}{2}} dt = \frac{1}{(k(k+d-2))^{s}} \int_{-1}^1 \tilde g_{d,s}(t) P_{d,k}(t) (1 - t^2)^{\frac{d-3}{2}}dt,
\end{equation}
where~$\tilde g_{d,s}$ is a bounded function on~$[-1, 1]$.
\end{lemma}
\begin{proof}
We use the following relation, derived in~\cite[Lemma 4]{bietti2021deep} for a function~$f_0$:
\begin{align*}
\int_{-1}^1 f_0(t) P_{d,k}(t) &(1-t^2)^{\frac{d-3}{2}} dt = \frac{1}{k(k+d-2)} \Big( -f_0(t)(1 - t^2)^{1 + \frac{d-3}{2}} P_{d,k}'(t) \Big|_{-1}^1 \\
	&\quad+ f_0'(t) (1 - t^2)^{1 + \frac{d-3}{2}} P_{d,k}(t) \Big|_{-1}^1  + \int_{-1}^1  f_1(t) P_{d,k}(t) (1 - t^2)^{(d-3)/2} dt \Big),
\end{align*}
where~$ f_1(t) = -f_0''(t) (1 - t^2) + (d - 1) t f_0'(t)$.
Note that the terms in brackets vanish when~$f_0$ and~$f_0'$ are bounded, and that~$f_1$ is~$2 s - 2$ times differentiable with bounded derivatives if~$f_0$ is~$2s$ times differentiable.
We may thus apply this recursively~$s$ times to~$f_0 = g$, with
\begin{align*}
    f_k(t) = - f_{k-1}''(t) (1 - t^2) + (d - 1) t f_{k-1}'(t),
\end{align*}
and we obtain the desired result, with~$\tilde g = f_s$.
\end{proof}

\begin{lemma}\label{lemma:uniform_P}
Let $g \in L^\infty([-1,1])$. It holds that
$$
\left| \int_{-1}^1  g(t) P_{d,k}(t) (1-t^2)^{(d-3)/2}\,dt  \right| \leq 2\pi d^{-1/2} \| g \|_\infty  \left( \frac{ d}{k} \right)^{(d-2)/2} ~.
$$
\end{lemma}
\begin{proof}
By \cite[equation (2.117)]{atkinson2012spherical}, we get that
\begin{align*}
\left| P_{d,k}(t) \right| & \leq \frac{1}{\sqrt{\pi}}\Gamma\left(\frac{d-1}{2}\right) \left(\frac{4}{k(1-t^2)}\right)^{(d-2)/2} \\
& \leq \frac{1}{\sqrt{\pi}}\left(\frac{d-1}{4}\right)^{(d-3)/2}\left(\frac{4}{k(1-t^2)}\right)^{(d-2)/2} \\
& \leq \frac{1}{\sqrt{\pi}}\left(\frac{d}{4}\right)^{-1/2}\left(\frac{d}{k(1-t^2)}\right)^{(d-2)/2} \leq  2d^{-1/2}\left(\frac{d}{k(1-t^2)}\right)^{(d-2)/2}~.
\end{align*}
Therefore it follows that
\begin{align*}
\left| \int_{-1}^1  g(t) P_{d,k}(t) (1-t^2)^{(d-3)/2}\,dt  \right|  & \leq 2 d^{-1/2} \| g \|_\infty  \left( \frac{ d}{k} \right)^{(d-2)/2} \int_{-1}^1 (1-t^2)^{-1/2}\,dt~,
\end{align*}
which concludes the proof.
\end{proof}

\begin{lemma}[Decay for~$\lambda \in (-1, 1)$]
\label{lemma:phi_interior}
Let $\phi_{\lambda+,\alpha}(t) := (t - \lambda)_+^\alpha \varphi_{\lambda+,\alpha}(t)$ and $\phi_{\lambda-,\alpha}(t) := (t - \lambda)_-^\alpha \varphi_{\lambda-,\alpha}(t)$, where $\varphi_{\lambda\pm,\alpha} \in C^\infty([-1,1])$ have support $(-1+\epsilon,1-\epsilon)$ for some $\epsilon > 0$ and take the value~$1$ at~$t=\lambda$. Then we have that 
\begin{equation}
\left| \int_{-1}^1 \phi_{\lambda\pm,\alpha}(t) P_{d,k}(t) dt \right| \leq C(d,\alpha) k^{-d/2 -\alpha + 3}~.
\end{equation}
Also let, for $\alpha$ integer, $\phi^*_{\lambda\pm,\alpha}(t) := (t-\lambda)_{\pm}^\alpha\log|t-\lambda| \varphi_{\lambda\pm,\alpha}^*(t)$, where $\varphi_{\lambda\pm,\alpha}^* \in C^\infty([-1,1])$ have support $(-1+\epsilon,1-\epsilon)$ for some $\epsilon > 0$ and take the value~$1$ at~$t=\lambda$. Then we have that
\begin{equation}
\left| \int_{-1}^1 \phi_{\lambda\pm,\alpha}^*(t) P_{d,k}(t) dt \right| \leq C(d,\alpha) k^{-d/2 -\alpha + 3}~.
\end{equation}
\end{lemma}
\begin{proof}
Let $\psi_{\lambda \pm, \alpha}(t) := \phi_{\lambda \pm, \alpha}(t) (1-t^2)^{-(d-3)/2}$. Notice that $\psi_{\lambda+,\alpha}$ satisfies the assumption of Lemma \ref{lemma:IbP} with $2s = 2\floor*{\frac{\alpha}{2}} \geq \alpha -2 $. Therefore we obtain
\begin{align*}
\left| \int_{-1}^1 \phi_{\lambda+,\alpha}(t) P_{d,k}(t) dt \right| &= \left| \int_{-1}^1 \psi_{\lambda+,\alpha}(t) P_{d,k}(t) (1 - t^2)^{(d-3)/2} dt \right| \\
  & \leq k^{-\alpha + 2} \left| \int_{-1}^1 \tilde{\psi}_{\lambda+,\alpha,d}(t) P_{d,k}(t) \left( 1 - t^2\right)^{(d-3)/2} dt \right| \\
& \leq C(d,\alpha) k^{-\alpha + 2 - d/2 + 1}~,
\end{align*}
where~$\tilde \psi_{\lambda+,\alpha,d}$ is a bounded function given by Lemma~\ref{lemma:IbP}, and
where we used Lemma \ref{lemma:uniform_P} to obtain the last inequality. The second inequality follows in the same way, by noticing that the function $\psi_{\lambda \pm, \alpha}^*(t) := \phi_{\lambda \pm, \alpha}^*(t) (1-t^2)^{-(d-3)/2}$  satisfies the assumption of Lemma \ref{lemma:IbP} with $2s = 2\floor*{\frac{\alpha-1}{2}} \geq \alpha -2 $ (since we assume that $\alpha$ is integer in this case). 
\end{proof}

\begin{lemma}[Decay for~$\lambda = \pm 1$]
\label{lemma:phi_pm1}
Let $\phi_{1,\alpha,s}(t) := (\frac{1 + t}{2})^{\alpha + s - \floor{\alpha}} (1 - t)^\alpha$, with~$\alpha$ non-integer and~$s$ integer. Then, $\phi_{1,\alpha,s}$ is~$s$ times differentiable at~$-1$ and obeys the decay
\begin{equation}
    \left| \int_{-1}^1 \phi_{1,\alpha,s}(t) P_{d,k}(t) dt \right| \leq C(d,\alpha,s) k^{-2(\alpha + 1)},
\end{equation}
where the constant~$C(d,\alpha,s)$ may be different depending on the parity of~$k$.

Similarly, let $\phi_{-1,\alpha,s}(t) := (\frac{1-t}{2})^{\alpha + s - \floor{\alpha}} (t + 1)^\alpha$,
then~$\phi_{-1,\alpha,s}$ is~$s$ times differentiable at~$1$, and obeys the same decay.
\end{lemma}
\begin{proof}
We begin by evaluating the decay of~$\psi_\alpha(t) := (1 - t^2)^\alpha$.
Following analogous calculations to~\cite[Lemma 6] {bietti2021deep}, we have\footnote{Note that while Lemma 6 in~\cite{bietti2021deep} is stated for $\alpha = \nu + \frac{d - 3}{2}$ for~$\nu > 0$, the derivation still holds for any~$\alpha > -1$.}
\begin{equation}
\label{eq:pm1_decay}
    \int \psi_\alpha(t) P_{d,k}(t) dt \sim C(d,\alpha) k^{-2(\alpha + 1)},
\end{equation}
for~$k$ even, and the integral is equal to zero for~$k$ odd.
We have
\begin{equation*}
    C(d,\alpha) = 2^{4 \alpha + 3} \frac{\omega_{d-2}}{\omega_{d-1}} \frac{\Gamma(\alpha + 1)^2}{\Gamma(2 \alpha + 2)} \frac{\Gamma(\alpha + \frac{3}{2}) \Gamma(\frac{d-1}{2}) \Gamma(\alpha + \frac{5-d}{2})}{\Gamma(-\frac{1}{2}) \Gamma(\alpha + \frac{7 - d}{2}) \Gamma(-\alpha + \frac{d-5}{2})},
\end{equation*}
where~$\omega_{p-1}$ is the surface of~$\Sbb^{p-1}$.

Now, let~$r := s - \floor{\alpha}$, so that we have
$$\phi_{1,\alpha,s}(t) = 2^{-\alpha - r} (1 + t)^r \psi_\alpha(t).$$
Let~$c_0, \ldots, c_r$ denote the coefficients of the degree-$r$ polynomial~$p(t) = 2^{-\alpha - r} (1 + t)^r$, so that
$$p(t) = 2^{-\alpha - r} (1 + t)^r = c_0 + c_1 t + \cdots + c_r t^r.$$
Using the relation (see, \eg,~\cite[Proposition 4.21]{costas2014spherical})
\begin{equation*}
    t P_{d,k}(t) = \frac{k}{2k + d - 2} P_{d,k-1}(t) + \frac{k + d - 2}{2k + d - 2} P_{d,k+1}(t),
\end{equation*}
we may then write
\begin{equation*}
    p(t) P_{d,k}(t) = \sum_{j=-r}^{r} b_{j}(k)P_{d,k+j}(t),
\end{equation*}
for some coefficients~$b_{j}(k)$ satisfying~$b_{j}(k) = O(1)$ as~$k \to \infty$.
Then, we have
\begin{align*}
    \int \phi_{1,\alpha,s} P_{d,k}(t) dt &= \int \psi_\alpha(t) p(t) P_{d,k}(t) dt \\
    &= \sum_{j=-r}^{r} b_{j}(k) \int \psi_\alpha(t) P_{d,k+j}(t) dt.
\end{align*}
When~$k \to \infty$, this is a sum of at most~$2r + 1 \leq 2 s + 1$ terms, each of which decays with~$k$ as~$k^{-2(\alpha + 1)}$ by~\eqref{eq:pm1_decay}. This yields the result.

The decay for~$\phi_{-1,\alpha,s}$ is proved analogously.
\end{proof}

\begin{lemma}[Cancelling singularities of the density]
\label{lemma:singularities}
Let~$q_\sigma$ denote the density of the random variable~$Z_\sigma = \langle \sigma \cdot x, x \rangle$, with~$\sigma \ne \mathrm{Id}$, and let $\bar \Lambda_\sigma$ be the set of eigenvalues of $\bar A_\sigma := (A_\sigma + A_\sigma^\top)/2$, where $A_\sigma$ is the permutation matrix of $\sigma$, and denote by~$\bar m_\lambda$ the multiplicity of~$\lambda \in \bar \Lambda_\sigma$.
Define
\begin{equation}
\label{eq:alpha_lambdas}
    \alpha_\lambda = \frac{d - \bar m_\lambda}{2} - 1.
\end{equation}
There exists constants~$\{c_{\lambda,i}\}$ such that the function defined by
\begin{align*}
   \tilde q_\sigma &= \sum_{\lambda \in \Lambda_\sigma \setminus \{1, \lambda_{\min}\}} \sum_{i = 0}^{\lceil d+1 - \alpha_\lambda \rceil} (c_{\lambda_+,i}\phi_{\lambda+,\alpha_\lambda + i} + c_{\lambda_-,i}\phi_{\lambda-,\alpha_\lambda + i}
   + c^*_{\lambda_+,i}\phi^*_{\lambda+,\alpha_\lambda + i} + c^*_{\lambda_-,i}\phi^*_{\lambda-,\alpha_\lambda + i}) \\
   &\quad + \mathbf{1}\{\lambda_{\min} > -1\}\sum_{i=0}^{\lceil d+1 - \alpha_{\lambda_{\min}} \rceil }  (c_{\lambda_{\min+},i}\phi_{\lambda_{\min+},\alpha_{\lambda_{\min}} + i} + c^*_{\lambda_{\min+},i}\phi^*_{\lambda_{\min+},\alpha_{\lambda_{\min}} + i}) \\
   &\quad + \sum_{\lambda \in \Lambda_\sigma \cap \{\pm 1\}} \sum_{i=0}^{\lceil d + 1+ \frac{d-3}{2} - \alpha_\lambda \rceil}  c_{\lambda,i} \phi_{\lambda,\alpha_\lambda+i,\ceil{d + 1+ \frac{d-3}{2}}}
   \end{align*}
satisfies that~$t \mapsto (q_\sigma(t) - \tilde q_\sigma(t))  (1 - t^2)^{- \frac{d-3}{2}}$ admits~$d+1$ bounded derivatives on~$[-1, 1]$.
\end{lemma}

\begin{proof}
Note that we have
\[\langle \sigma \cdot x, x \rangle = \frac{1}{2}(\langle A_\sigma x, x \rangle + \langle x, A_\sigma, x \rangle) = \langle \bar A_\sigma x, x \rangle,\]
where~$\bar A_\sigma = \frac{1}{2}(A_\sigma + A_\sigma^\top)$ is symmetric and thus has real eigenvalues.
When~$A_\sigma$ is a permutation matrix, these eigenvalues are in~$[-1, 1]$, as the real part of complex roots of unity.

We then identify the singularities of the density~$q_\sigma$, which are the same as those of the cumulative distribution function, up to one fewer degree of smoothness.
Such singularities are shown in the following lemma, proved in \cite{saldanha1996accumulated}.

\begin{lemma}[\cite{saldanha1996accumulated}]
Consider $\sigma \in G$ as above, and let $\bar \Lambda_\sigma$ be the set of eigenvalues of $\bar A_\sigma$. For each $\lambda \in \bar \Lambda_\sigma$, we denote by $\bar m_\lambda$ its multiplicity. Then the cumulative distribution function~$Q_\sigma$ of~$Z_\sigma = \langle \sigma \cdot x, x \rangle$ takes the form
$$
Q_\sigma(t) = \varphi(t) + \sum_{\lambda \in \bar \Lambda_\sigma} g_\lambda(t) 
$$
where $\varphi$ is analytic and 
\begin{itemize}[leftmargin=0.4cm]
\item $g_\lambda(t) = | t - \lambda | (t - \lambda)^{\frac{d - \bar m_\lambda}{2} - 1} \varphi_\lambda^1(t) + (t - \lambda)^{\frac{d - \bar m_\lambda}{2}} \log (| t - \lambda |) \, \varphi_\lambda^2(t) $ if $d - \bar m_\lambda$ is even,  
\item $g_\lambda(t) =  (t - \lambda)^{\frac{d - \bar m_\lambda}{2}}_+ \varphi_\lambda^1(t) + (\lambda - t)^{\frac{d - \bar m_\lambda}{2}}_+ \, \varphi_\lambda^2(t) $ if $d - \bar m_\lambda$ is odd,  
\end{itemize}
for some $\varphi^1_\lambda,\varphi^2_\lambda$ analytic. 
Further, the term involving~$\log(|t - \lambda|)$ only appears for~$\lambda \in (-1, 1)$.

\end{lemma}

In particular, it follows from this lemma by differentiation that we may write
$$q_\sigma(t) = \tilde \varphi(t) + \sum_{\lambda \in \bar \Lambda_\sigma} \tilde g_\lambda(t),
$$
with~$\tilde \varphi$ analytic and
\begin{itemize}[leftmargin=0.4cm]
    \item $\tilde g_\lambda(t) = (t - \lambda)_+^{\frac{d-\bar m_\lambda}{2} - 1} \tilde \varphi_{\lambda,1}(t) + (t - \lambda)_-^{\frac{d-\bar m_\lambda}{2} - 1} \tilde \varphi_{\lambda,2}(t) + (t - \lambda)_+^{\frac{d-\bar m_\lambda}{2} - 1} \log(|t - \lambda|) \tilde \varphi_{\lambda,3}(t) + (t - \lambda)_-^{\frac{d-\bar m_\lambda}{2} - 1} \log(|t - \lambda|) \tilde \varphi_{\lambda,4}(t)$, if $d - \bar m_\lambda$ is even
    \item $\tilde g_\lambda(t) = (t - \lambda)_+^{\frac{d-\bar m_\lambda}{2} - 1} \tilde \varphi_{\lambda,1}(t) + (t - \lambda)_-^{\frac{d-\bar m_\lambda}{2} - 1} \tilde \varphi_{\lambda,2}(t)$, if $d - \bar m_\lambda$ is odd,
\end{itemize}
where~$\tilde \varphi_{\lambda,i}$ are analytic for~$i \in [4]$.

The result then follows by appropriately ``cancelling'' those singularities up to order~$d$ using the simple functions~$\phi$ introduced in the previous lemmas, and noting that for singularities at~$\pm 1$, we require an additional~$(d-3)/2$ degrees of smoothness, so that we may divide by the weight function~$(1 - t^2)^{(d-3)/2}$.

For instance, for~$\lambda \in (-1, 1)$, an appropriate exponent~$\alpha$ and an analytic~$\varphi$, we may write
\begin{align}
    (t - \lambda)_+^{\alpha} \varphi(t) &= (t - \lambda)_+^{\frac{d-\bar m_\lambda}{2} - 1} (c_0 + (t - \lambda) \psi(t)),
\end{align}
with~$\psi$ analytic. Then for~$\phi_{\lambda+,\alpha}$ as in Lemma~\ref{lemma:phi_interior}, we have
\begin{equation*}
    (t - \lambda)_+^{\alpha} \varphi(t) - c_0 \phi_{\lambda+,\alpha}(t) = (t - \lambda)_+^{\alpha + 1} \tilde \varphi(t),
\end{equation*}
with~$\tilde \varphi$ analytic.
We may then repeat this process with functions~$\phi_{\lambda+,\alpha + 1}$, $\phi_{\lambda+,\alpha+2}$, etc., to finally obtain that
\begin{equation*}
    (t - \lambda)_+^{\alpha} \varphi(t) - \sum_{i=0}^{\ceil{d+1 - \alpha}} c_i \phi_{\lambda+,\alpha + i}(t)
\end{equation*}
is~$d$ times differentiable, as desired.
A similar reasoning can be applied for other types of singularities.
For the terms involving~$\log(|t-\lambda|)$, which only appear for~$\lambda \in (-1, 1)$, note that the corresponding exponent~$\alpha_\lambda$ is integer since~$d - \bar m_\lambda$ is even, so that Lemma~\ref{lemma:phi_interior} applies.

\end{proof}

We are now ready to state the proof of Proposition~\ref{prop:decay}.

\begin{proof}[Proof of Proposition~\ref{prop:decay}]
Let~$q_\sigma$ be the density of~$\langle \sigma \cdot x, x \rangle$, and let~$\tilde q_\sigma$ be as in Lemma~\ref{lemma:singularities}.
By Lemma~\ref{lemma:singularities} and Lemma~\ref{lemma:IbP},
we have
\begin{align*}
    \int_{-1}^1 (q_\sigma(t) - \tilde q_\sigma(t)) P_{d,k}(t) dt
    = \int_{-1}^1 \frac{q_\sigma(t) - \tilde q_\sigma(t)}{(1 - t^2)^{\frac{d-3}{2}}} P_{d,k}(t) (1 - t^2)^{\frac{d-3}{2}} dt 
    \leq C k^{-d},
\end{align*}
where we have bounded the integral on the r.h.s.~of~\eqref{eq:ibp} by a constant.
Renaming the terms in~$\tilde q$ as~$\tilde q_\sigma = \sum_i c_i q_i$, where each~$q_i$ is as in Lemma~\ref{lemma:phi_interior} or Lemma~\ref{lemma:phi_pm1}, we have
\begin{equation*}
    \gamma_{d,\sigma}(k) \leq \sum_i c_i \int q_i P_{d,k} + O(k^{-d}).
\end{equation*}
Now, note that the eigenvalues of~$\bar A_\sigma = \frac{1}{2}(A_\sigma + A_\sigma^\top)$ are the real parts of the complex eigenvalues of the permutation matrix~$A_\sigma$.
Since~$A_\sigma$ is real, its complex eigenvalues come in conjugate pairs, so that for any~$\lambda \in \Lambda_\sigma$ with~$\lambda \notin \R$, we have~$ m_\lambda = \bar m_{Re(\lambda)} / 2$, where~$\bar m$ are the multiplicities of Lemma~\ref{lemma:singularities}.
Note that in the case of permutations, the eigenvalues of~$A_\sigma$ are roots of unity, so that we have
\[
m_\lambda = \begin{cases}
\bar m_\lambda, &\text{ if }\lambda \in \{\pm 1\} \\
\bar m_{Re(\lambda)}/2, &\text{ otherwise.}
\end{cases}
\]

The result then follows by applying the decays given by Lemma~\ref{lemma:phi_interior} and Lemma~\ref{lemma:phi_pm1} to each component~$q_i$ with the appropriate~$\alpha_\lambda$, and focusing on the leading-order terms.
Namely, for~$\lambda \in \{\pm 1\}$, by Lemma~\ref{lemma:phi_pm1}, the leading order term in~$\gamma_{d,\sigma,\lambda}(k)$ has decay
\[
k^{-2(\alpha_\lambda + 1)} = k^{-2(\frac{d - \bar m_\lambda}{2} - 1 + 1)} = k^{-d + m_\lambda},
\]
while for~$\lambda \notin \{\pm 1\}$, we have~$Re(\lambda) \in (-1, 1)$, hence using~\eqref{eq:alpha_lambdas} and~$\bar m_{Re(\lambda)} = 2 m_\lambda$, we have~$\alpha_{Re(\lambda)} = \frac{d}{2} - m_\lambda - 1$, so that the decay, by Lemma~\ref{lemma:phi_interior}, is upper bounded by
\[
k^{-\frac{d}{2} - \alpha_{Re(\lambda)} + 3} \leq k^{-d + m_\lambda + 4}.
\]
This concludes the proof.
\end{proof}

\subsection{Proof of Corollary~\ref{corollary:gamma_d_sigma_decay} (leading order of~$\gamma_{d,\sigma}$)}
\begin{proof}
By Proposition \ref{prop:decay}, we get that
$$
\gamma_{d,k}(\sigma) \lesssim k^{-d + s}
$$
where $s = \max_{\lambda \in \Lambda_\sigma} \left\{ m_\lambda + 4 \cdot \mathbf{1}(|\lambda| < 1) \right\}$. Notice that, for any permutation $\sigma$, it holds $m_1 \geq m_\lambda$. 
Therefore, we have
$$
s \leq m_1 + 4 \cdot\mathbf{1}\{ \exists |\lambda| < 1 ~:~  m_1 < m_\lambda + 4 \}~.
$$
Now, if $m_1 < m_\lambda + 4$ for some~$|lambda| < 1$, since $m_1 + m_\lambda \leq d$, it must hold $2 m_1 \leq d + 3$, or, equivalently, $m_1 \leq d/2 + 3/2$.
It follows that 
$$
s \leq \begin{cases} m_1 & ~\text{ if $m_1 > (d+3)/2$}~,\\ d/2 + 5.5 & ~\text{ otherwise.} \end{cases} 
$$
This concludes the proof, since $m_1 = c(\sigma)$.
\end{proof}

\subsection{Proof of Corollary~\ref{theo:number_of_cycles}  (different upper bounds using permutation statistics)}

\begin{proof}
For all $\sigma \in G\setminus \zeta(G,s)$, it holds $c(\sigma) \leq s$ and
$$
\gamma_{d,\sigma}(k) \lesssim k^{-d + \eta_\sigma}~,
$$
with
$$
\eta_\sigma = c(\sigma) \cdot \mathbf{1}\left( c(\sigma) > (d+3)/2 \right) + \left( \frac{d}{2}+6 \right) \mathbf{1}\left( c(\sigma) \leq (d+3)/2 \right)~.
$$
In particular, we have
$$
\gamma_d(k) = \frac{\zeta(G,s)}{|G|} + O\left( k^{-d+\eta} \right)
$$
where 
$$
\eta = \max_{c(\sigma) \leq s} \eta_\sigma~.
$$
Denote $s^*(s) = \max_{\sigma \in G \setminus \zeta(G,s)} c(\sigma)$. If $s^*(s) \geq d/2 + 6$, then it holds that $s^*(s) = \eta$. Otherwise, $c(\sigma) \leq d/2 + 7$ for any $\sigma$ such that $c(\sigma) \leq s$, which implies that $s \leq d/2 + 6$. It follows that 
$$
\eta = \max\{ s^*(s), \;d/2 + 6 \} \leq \max\{ s, \;d/2 + 6 \}.
$$
\end{proof}

\subsection{Details on Example \ref{example:full_permutation_group} (full permutation group)}

The number of permutations in $G = S_d$ which fix exactly $n$ elements is given by $\binom{n}{k} \,!(n-k)$, where $!m$ denotes the $m$-th subfactorial:
$$
!m := \floor*{\frac{m! + 1}{e}} \leq \frac{2m!}{e}~.
$$
It follows that
\begin{align*}
\frac{\xi(G,s)}{|G|} & = \frac{1}{d!}\sum_{k = s+1}^d \frac{d! }{(d-k)! k!} !(d-k) \leq \frac{2}{e} \sum_{k = s+1}^d \frac{1}{k!} \\
& \leq \frac{2}{e (s+1)!} \sum_{k = s+1}^d \frac{1}{(s+2)^{k - (s+1)}} = \frac{2}{e (s+1)!} \frac{1 - \frac{1}{(s+2)^{d-s}}}{1 - \frac{1}{s+2}} \\
& \leq \frac{2 (s+2)}{e (s+1)} \frac{1}{(s+1)!} \leq \frac{2}{(s+1)!}~.
\end{align*}

\section{Proofs for Section~\ref{sec:stability}}
\label{sec:appx_stability}

\subsection{Proof of Lemma~\ref{lemma:spectral_stability} (spectral properties of smoothing operator~$S_G$)}

\begin{proof}
As in the invariant case, we note that for any degree~$k$, the space~$V_{d,k}$ of spherical harmonics of degree~$k$ is stable by~$S_G$, \ie, $S_G V_{d,k} \subset V_{d,k}$.
Since~$S_G$ is self-adjoint, we may then find an orthonormal basis of such spherical harmonics, which we denote~$\Ybar_{k,j}$, for~$j = 1, \ldots, N(d,k)$, such that the restriction of~$S_G$ to~$V_{d,k}$ is diagonal, and we have~$S_G \Ybar_{k,j} = \lambda_{k,j} \Ybar_{k,j}$, with~$\lambda_{k,j} \geq 0$.

It remains to show~\eqref{eq:gammadk_stability}.
Define the operator~$\Pi_k f = \EE{y}{P_{d,k}(\langle \cdot, y \rangle) f(y)}$.
$S_G \Pi_k$ is then an integral operator with kernel
\begin{equation}
\label{eq:stab_smoothing_kernel_1}
H(x, y) = \sum_{\sigma \in G} h(\sigma) P_{d,k}(\langle \sigma \cdot x, y \rangle).
\end{equation}
Since~$\Ybar_{k,j}$, $j = 1, \ldots, N(d,k)$ forms an orthonormal basis of~$V_{d,k}$, by the addition formula of spherical harmonics, we have
\begin{align*}
\Pi_k = \frac{1}{N(d,k)} \sum_{j=1}^{N(d,k)} \Ybar_{k,j} \Ybar_{k,j}^*.
\end{align*}
It follows that
\begin{align*}
S_G \Pi_k = \frac{1}{N(d,k)} \sum_{j=1}^{N(d,k)} S_G \Ybar_{k,j} \Ybar_{k,j}^* = \frac{1}{N(d,k)} \sum_{j=1}^{N(d,k)} \lambda_{k,j} \Ybar_{k,j} \Ybar_{k,j}^*.
\end{align*}
This implies that the kernel~$H$ of the operator~$S_G \Pi_k$ can also be expressed as
\begin{equation}
\label{eq:stab_smoothing_kernel_2}
H(x,y) = \frac{1}{N(d,k)} \sum_{j=1}^{N(d,k)} \lambda_{k,j} \Ybar_{k,j}(x) \Ybar_{k,j}(y).
\end{equation}
Fixing~$y = x$ and taking expectations over~$x \sim \tau$ in both~\eqref{eq:stab_smoothing_kernel_1} and~\eqref{eq:stab_smoothing_kernel_2} proves the equality.

\end{proof}

\subsection{Proof of Theorem~\ref{thm:generalization_stability} (generalization with geometric stability)}

The proof of the theorem is analogous to that of Theorem~\ref{thm:generalization}, replacing the control on~$\Ncal_\Kbar(\lambda)$ with that of Lemma~\ref{lemma:dof_stable} below, which provides an extension of Lemma~\ref{lemma:dof} to generic smoothing operators, at the cost of a weaker constant~$\nu_d(\ell)^{1/\alpha} $ instead of~$\nu_d(\ell)$.
The remark on the constant~$C_4$ being potentially smaller for the kernel~$K$ stems from the fact that we no longer have equal approximation errors for the two kernels, and that the quantity~$C_{f^*}$ in~\eqref{eq:krr_generic_bound} in this case is the one given by the source condition (A2) instead of (A6), which is smaller, as we now show. Indeed, note that if~$f^* = S_G^r T_K^r g$, then we have~$f^* = T_K^r (S_G^r g) = T_K^r \tilde g$ with~$\tilde g = S_G^r g$, since~$S_G$ and~$T_K$ are diagonalized in the same basis and hence commute.
The result follows by noting that we have~$\|\tilde g\|_{L^2(d \tau)} \leq \|g\|_{L^2(d \tau)}$, since~$\|S_G^r\| \leq \|S_G\|^r \leq 1$ (indeed we have the operator norm bound~$\|S_G\| \leq 1$, which follows from a simple triangle inequality).

\begin{lemma}[Degrees of freedom for~$\Kbar$ with stability.]
\label{lemma:dof_stable}
Assume (A5). We have
\begin{equation*}
\Ncal_K(\lambda) \leq C_K \lambda^{-1/\alpha},
\end{equation*}
and for any~$\ell \geq 0$, we have
\begin{equation}
\Ncal_\Kbar(\lambda) \leq D(\ell) + \nu_d(\ell)^{1/\alpha} C_K \lambda^{-1/\alpha},
\end{equation}
with the same constant~$C_K$.
\end{lemma}
\begin{proof}
The first statement is a standard consequence of Assumption~(A5).
Namely, if~$\xi_m$ denote the eigenvalues of~$T_K$ (namely, the same as~$\mu_k$ counted with their multiplicities) and~$\xi_m \leq C(m + 1)^{-\alpha}$, we have
\begin{align*}
\Ncal_K(\lambda) &= \sum_{m \geq 0} \frac{\xi_m}{\xi_m + \lambda} \\
	&\leq \sum_{m \geq 0} \frac{1}{1 + \lambda C^{-1}(m+1)^{\alpha}} \\
	&\leq \int_{0}^\infty \frac{dt}{1 + \lambda C^{-1} t^\alpha} \\
	&\leq \frac{C^{1/\alpha}\lambda^{-1/\alpha}}{\alpha} \int_{0}^\infty \frac{u^{1/\alpha - 1} du}{1 + u}
	= C_K \lambda^{-1/\alpha},
\end{align*}
with~$C_K := \frac{C^{1/\alpha}}{\alpha} \int_0^\infty \frac{u^{1/\alpha - 1} du}{1 + u}$.

We now write
\begin{align*}
\Ncal_\Kbar(\lambda) &= \sum_{k \geq 0} \sum_{j=0}^{N(d,k)} \frac{\lambda_{k,j} \mu_k}{\lambda_{k,j} \mu_k + \lambda} \\
	&= \sum_{k \geq 0} \sum_j \frac{\lambda_{k,j}}{\lambda_{k,j} + \lambda \mu_k^{-1}} \\
	&\leq \sum_k N(d,k) \frac{\bar \lambda_k}{\bar \lambda_k + \lambda \mu_k^{-1}} \quad \text{ (by Jensen's inequality, with $\bar \lambda_k = N(d,k)^{-1} \sum_j \lambda_{k,j}$)} \\
	&\leq \sum_k N(d,k) \frac{\bar \lambda_k \mu_k}{\bar \lambda_k \mu_k + \lambda},
\end{align*}
We may then write, for some~$\ell \geq 1$,
\begin{equation*}
\Ncal_\Kbar(\lambda) \leq D(\ell) + \sum_{k \geq 0} \frac{\bar \mu_k}{\bar \mu_k + \lambda},
\end{equation*}
where
\begin{equation*}
\bar \mu_k = \begin{cases}
	\bar \lambda_k \mu_k, &\text{ if }k \geq \ell\\
	0, &\text{ o/w.}
\end{cases}
\end{equation*}
Note that for~$k \geq \ell$, we have~$\bar \lambda_k = \frac{\sum_j \lambda_{k,j}}{N(d,k)} = \gamma_d(k) \leq \nu_d(\ell)$, and the same holds trivially for~$k < \ell$.
Then, writing~$\bar \xi_m$ the collections of~$\bar \mu_k$ counted with multiplicities, we may write~$\bar \xi_m \leq \nu_d(\ell) C (m + 1)^{-\alpha}$, with the same constant~$C$ as in~(A5).
Repeating the argument above for bounding~$\Ncal_\Kbar(\lambda)$ in terms of~$\lambda^{-1/\alpha}$ then yields the result.
\end{proof}

\subsection{Proof of Proposition \ref{prop:defgroup} (upper bound on~$\gamma_d(k)$ for deformations)}

\begin{proof}
We first show that~$\Phi_2$ is stable under inversion, and later proceed to study lower bounds on its number of elements, and cycle statistics.

\paragraph{Step 1: $\Phi_2^{-1} = \Phi_2$.}
Let us first establish that $\Phi_2$ is closed under inversion. 

First observe that 
\begin{equation}
\label{eq:phi_equiv}
\Phi_2 = \{ \sigma; |\sigma(u+1) - \sigma(u) - 1| \leq 2 ~\forall~u\}~.    
\end{equation}

The inclusion $\mathrm{LHS} \subseteq \mathrm{RHS}$ is immediate by definition. The reverse inclusion is obtained by the triangle inequality, by observing that if $u < \tilde{u} < u'$, then 
\begin{eqnarray*}
| \sigma(u) - \sigma(u') - (u-u') | &=& | \sigma(u) - \sigma(\tilde{u}) - (u - \tilde{u}) + \sigma(\tilde{u}) - \sigma(u') - (\tilde{u} - u') | \\ 
&\leq& | \sigma(u) - \sigma(\tilde{u}) - (u - \tilde{u})| + | \sigma(\tilde{u}) - \sigma(u') - (\tilde{u} - u') | ~,
\end{eqnarray*}
so by induction if the condition holds for small pairs $(u, \tilde{u})$, $(\tilde{u}, u')$ it extends to all pairs $(u,u')$.

We directly verify from (\ref{eq:phi_equiv}) that $\sigma \in \Phi_2$ iff it holds
\begin{equation}
\label{eq:bubu}
\forall~u~,~\sigma(u+1) = \sigma(u) + \{3, 2, 1, -1\}~,    
\end{equation}
since we need to have $\sigma(u) \neq \sigma(u')$ whenever $u\neq u'$.

Let now $\tilde{u} = \sigma(u)$, so $\sigma^{-1}(\tilde{u})=u$. We will show that $\sigma^{-1}$ also verifies (\ref{eq:bubu}). 
We want to enumerate all possible $u'$ so that $\sigma(u') = \tilde{u}+1$. Clearly $\sigma^{-1}(\tilde{u}+1) \neq \sigma^{-1}(\tilde{u})$ so $u' \neq u$. 

Suppose by contradiction that~$u' < u - 1$.
Note that we must have~$\sigma(u' + 1) \geq \tilde u + 2$ since~$\tilde u$ and~$\tilde u + 1$ already have pre-images (namely~$u$ and~$u'$), and smaller values would violate~$\sigma(u' + 1) - \sigma(u') \in \{3, 2, 1, -1\}$.
Similarly~$\sigma(u' + s) \geq \tilde u + 2$ for all~$s = 2, \ldots, u - u' - 1$, since otherwise we would need a step~$\sigma(u' + s + 1) - \sigma(u' + s) \leq -3$, which is ruled out by~\eqref{eq:bubu}.
Then it must be that~$\sigma(u) - \sigma(u - 1) = \tilde u - \sigma(u-1) \leq -2$, which is a contradiction.
We have thus shown~$u' \geq u - 1$.

Similarly, let us show~$u' \leq u + 3$.
Assume, by contradiction that~$u' > u + 3$.
Note first that the only way to have~$\sigma(u + s) < \tilde u$ for some~$s \in [0, u' - u]$ is to only have~$\sigma(u + 1) = \tilde u - 1$, and~$\sigma(u + s) \geq \tilde u + 2$ otherwise.
Indeed, values smaller than~$\tilde u$ must happen just following~$u$ in order to allow decreasing by 1, and having additional negative steps after~$u + 1$ (\eg,~$\sigma(u + 2) = \tilde u - 2$) would require a step~$\sigma(u + s + 1) - \sigma(u + s) > 3$ for some~$s \in [2, u' - u]$ (since the values $\tilde u - 1, \tilde u, \tilde u + 1$ already have pre-images, given by~$u + 1, u, u'$, respectively), which is a contradiction.
Then, if~$\sigma(u + 1) = \tilde u - 1$, we must have~$\sigma(u + 2) = \tilde u + 2$ (since we cannot have longer steps), which implies~$\sigma(u' - 1) \geq \tilde u + 3$ and thus~$\sigma(u') - \sigma(u'-1) \leq -2$, which is a contradiction.
Alternatively, we must have~$\sigma(u + s) \geq \tilde u + 2$ for all~$s \in [1, u' - u - 1]$.
This implies~$\sigma(u' - 1) = \tilde u + 2$, $\sigma(u' - 2) = \tilde u + 3$, and more generally~$\sigma(u' - t) = \tilde u + t + 1$, since these are the only allowed steps to obtain~$\sigma(u') = \tilde u + 1$.
Then, we have~$\sigma(u + 1) = \sigma(u' - (u' - u - 1)) = \tilde u + (u' - u) > \tilde u + 3$, which is in contradiction with~$\sigma(u + 1) - \sigma(u) \leq 3$.
We have thus proved~$u' \leq u + 3$.
We thus have that~$\sigma^{-1}$ satisfies~\eqref{eq:bubu}, which shows~$\Phi_2^{-1} = \Phi_2$.

\paragraph{Step 2: Lower bound on $|\Phi_2|$.}
Denote as before $\sigma(u) = \tilde{u}$. 
By denoting $\Delta_{v} = \sigma(v) - v$ for arbitrary $v$, observe that $\Delta_{u'} - \Delta_{u'-1} \in \{ 2, 1, 0, -1\}$.
Similarly we define $\Gamma_{\tilde{u}} := \sigma^{-1}(\tilde{u}) - \tilde{u}$. By the previous argument, we also have $\Gamma_{\tilde{u}+1} = \Gamma_{\tilde{u}} + \{ 2, 1, 0, -1\}$. 
Fix an arbitrary $u_0$, say $u_0=1$ and consider the subset of $\Phi_2$ given by
$$\Phi_2^{b} = \{ \sigma \in \Phi_2; \sigma(u_0) = u_0, \sigma(u_0-1) = u_0 - 1  \}~.$$ 
$\Phi_2^b$ thus contains permutations with `fixed' boundary conditions. 
For $\sigma \in \Phi_2^{b}$, the boundary condition prevents $\Delta_{u_0+1} < \Delta_{u_0}$, so we identify the following possible cases:
\begin{enumerate}[leftmargin=1.5cm]
\item[1-block:] $\Delta_{u_0+1} = \Delta_{u_0}$. In this case, $\sigma(u_0+1) = \sigma(u_0) +1$. 
    \item[2-block:] $\Delta_{u_0+1} = \Delta_{u_0} + 1$. This implies $\Gamma_{\tilde{u_0}+2} = \Gamma_{\tilde{u_0}} -1$, which in turn implies $\Gamma_{\tilde{u_0}+1} = \Gamma_{\tilde{u}_0}+1$, and finally $\Delta_{u_0+2} = \Delta_{u_0+1} -2$. In summary, $\sigma(u_0+1) = \sigma(u_0) + 2 $ and $\sigma(u_0+2) = \sigma(u_0) + 1$.
    
    \item[3-block:] $\Delta_{u_0+1} = \Delta_{u_0} + 2$. This implies $\Gamma_{\tilde{u_0}+3 } = \Gamma_{\tilde{u_0}} -2$, which necessarily implies $\Gamma_{\tilde{u_0}+1 } = \Gamma_{\tilde{u_0}} +2$, $\Gamma_{\tilde{u_0}+2 } = \Gamma_{\tilde{u_0}+1} -2$ and  $\Gamma_{\tilde{u_0}+3 } = \Gamma_{\tilde{u_0}+2} -2$. This corresponds to $\sigma(u_0+1) = \sigma(u_0)+3 $, $\sigma(u_0+2) = \sigma(u_0)+2 $ and $\sigma(u_0+3) = \sigma(u_0)+1 $. 
\end{enumerate}

So an element of $\Phi_2^{b}$ can be constructed sequentially by assembling three possible `blocks' $B_i$ of size $i=\{1,2,3\}$. 
Moreover, we verify immediately that the following transitions are admissible:
$$B_1 \to B_{\{1,2,3\}}~,~B_2 \to B_{\{1,2 \}}~,~B_3 \to B_1~.$$
Thus, by denoting $\mathcal{B}(m; B_i)$ the number of permutations in $\Phi_2^{b}$ restricted to their first $m$ elements, and which that start (after $u_0$) with a block of type $B_i$, we have the following recursion:
\begin{eqnarray}
\label{eq:goodsystem}
    \mathcal{B}(m; B_1) &=& \mathcal{B}(m-1; B_1) + \mathcal{B}(m-1; B_2) + \mathcal{B}(m-1; B_3) \nonumber \\ 
    \mathcal{B}(m; B_2) &=& \mathcal{B}(m-2; B_2) + \mathcal{B}(m-2; B_1) \nonumber \\
    \mathcal{B}(m; B_3) &=& \mathcal{B}(m-3; B_1) ~, 
\end{eqnarray}
with $\mathcal{B}(i; B_i) = 1$. 
Let $F_i(z) := \sum_{m\geq 0} \mathcal{B}(m, B_i) z^m$ be the generating function associated to each of the above sequences. 
We have 
$$F_1(z) = z^{-1} (F_1(z) + F_2(z) + F_3(z))~,~F_2(z) = z^{-2}( F_1(z) + F_2(z))~,~F_3(z) = z^{-3} F_1(z)~.$$
By substituitng $F_2, F_3$ into the first equation, we obtain
$$ F_1(z) (1 - z^{-1} - z^{-3} (1 - z^{-2})^{-1} - z^{-4}  ) = 0~,$$
so $F_1$ has a pole at $\tau \approx 1.714 $, the solution of the associated characterstic equation 
$z = 1 + \frac{1}{z^2 - 1} + \frac{1}{z^3}$. Moreover, this pole is also present in $F_2$ and $F_3$. This shows that $\mathcal{B}(m, B_i) \asymp C_i \tau^m $, and hence 
$$|\Phi_2 | \geq |\Phi_2^{b}| = \sum_{i=1}^3 \mathcal{B}(d; B_i) = \Theta(\tau^d)~.$$

\paragraph{Step 3: cycle statistics.}
Let us now compute a bound for $\gamma_d(k)$ using Corollary \ref{theo:number_of_cycles}. 
Let 
$$\xi(\Phi_2, n; d)= \left\{ \sigma \in \Phi_2~:~\mathrm{Fix}(\sigma) \geq n\right\}$$
denote the set of elements of $\Phi_2$ that fix at least $n$ positions, set $n=(1-\eta) d$, and assume $\eta < 1/2$.  
Observe that this necessarily implies that two consecutive indices, say $u_0$ and $u_0-1$, are fixed, by the pigeonhole principle. Thus 
$$\xi(\Phi_2, n) \subset \Phi_2^b~,$$ 
and we can use the characterisation of elements in $\Phi_2^b$. 
We have 
\begin{align*}
    | \xi(\Phi_2, (1-\eta)d; d) | &\leq \sum_{n'=(1-\eta)d}^d \binom{d}{d-n'}|\mathcal{B}(n'; B_1)+\mathcal{B}(n'; B_2)+\mathcal{B}(n'; B_3)| \\
    &\leq C \tau^{\eta d} \sum_{n'=(1-\eta)d}^d \binom{d}{d-n'} \leq C \left(e \eta^{-1}\right)^{\eta d} \tau^{\eta d}~,
\end{align*}
where $C$ is an abolute constant. 
Finally, from Example~\ref{example:full_permutation_group}, we have
$$n < c(\sigma) \Rightarrow 2n - d < \mathrm{Fix}(\sigma)~,$$
thus $\zeta(\Phi_2,n) \leq \xi(\Phi_2, 2n-d)$. 
By picking $n = (1-\eta) d $ with $\eta < 1/4$, 
we have $2n - d = (1-2\eta)d > d/2$ and 
\begin{eqnarray}
\gamma_d(k) &\leq& \frac{\xi(\Phi_2,(1-2\eta)d)}{|\Phi_2|} + O \left(k^{-d + \max(n,d/2+7)}\right)~\\
&\leq& C (e (2\eta)^{-1})^{2\eta d} \tau^{(2\eta-1)d} + O \left(k^{-d + \max(n,d/2+7)}\right)~. \\
&=& C \left(\frac{e^{2\eta}}{(2\eta)^{2\eta}\tau^{1-2\eta}} \right)^d + O \left(k^{-\eta d}\right)~.
\end{eqnarray}
When $2\eta < 0.15$, we verify that 
$\frac{e^{2\eta}}{(2\eta)^{2\eta}\tau^{1-2\eta}} < 1$.

\end{proof}

\end{document}